\documentclass{article}

\usepackage{microtype}
\usepackage{graphicx}
\usepackage{subfigure}
\usepackage{booktabs} 

\usepackage{breakurl}

\usepackage{hyperref}

\usepackage[accepted]{icml2024}

\usepackage{amsmath}
\usepackage{amssymb}
\usepackage{mathtools}
\usepackage{amsthm}
\usepackage{dsfont}

\usepackage[capitalize,noabbrev]{cleveref}

\usepackage{breakurl}

\DeclareMathOperator{\Hom}{Hom}

\DeclareMathOperator{\Aut}{Aut}

\usepackage{mathdots}
\usepackage{bm}
\usepackage{nicematrix}
\usepackage{xcolor}

\usepackage{tabularray}
\usepackage{comment}

\usepackage[title]{appendix}
\usepackage{appendix}
\usepackage{xurl}

\usepackage{ytableau}
\usepackage{tikzit}

\tikzstyle{node}=[fill=white, draw=black, shape=circle, minimum size=1mm, ultra thick]
\tikzstyle{red node}=[fill=white, draw=red, shape=circle, minimum size=1mm, ultra thick]
\tikzstyle{green node}=[fill=white, draw={rgb,255: red,0; green,138; blue,0}, shape=circle, minimum size=1mm, ultra thick]
\tikzstyle{small box}=[fill=white, draw=black, shape=rectangle, minimum height=0.5cm, minimum width=0.5cm, ultra thick]
\tikzstyle{weyl}=[fill=white, draw={rgb,255: red,0; green,0; blue,109}, shape=rectangle, minimum height=0.5cm, minimum width=0.5cm]
\tikzstyle{filled_node}=[fill=black, draw=black, shape=circle, minimum size=1mm, ultra thick]

\tikzstyle{thick}=[-, ultra thick]
\tikzstyle{blue_thick}=[-, ultra thick, draw=blue]
\tikzstyle{dashes}=[-, dashed, draw={rgb,255: red,191; green,191; blue,191}, dash pattern=on 2mm off 1mm, fill={rgb,255: red,244; green,228; blue,0}]
\tikzstyle{thick_arrow}=[ultra thick, ->]
\tikzstyle{dash_1}=[-, dashed]
\tikzstyle{dash_2}=[-, dashed, fill={rgb,255: red,246; green,235; blue,255}]
\tikzstyle{dash_3}=[-, dashed, fill={rgb,255: red,229; green,255; blue,181}]
\tikzstyle{dash_4}=[-, dashed, fill={rgb,255: red,255; green,209; blue,153}]
\tikzstyle{red_thick}=[-, ultra thick, draw=red]
\tikzstyle{dash_5}=[-, dashed, fill={rgb,255: red,225; green,255; blue,254}]
\tikzstyle{jellyfish}=[-, ultra thick, fill={rgb,255: red,34; green,48; blue,255}]
\tikzstyle{arrow}=[thick, ->]
\tikzstyle{thick_red_arrow}=[ultra thick, ->, draw=red]
\tikzstyle{thick_blue_double_arrow}=[ultra thick, <->, draw=blue]
\tikzstyle{thick_green_double_arrow}=[ultra thick, <->, draw={rgb,255: red,0; green,138; blue,0}]

\usepackage{graphicx}
\usepackage[most]{tcolorbox}

\definecolor{melon}{RGB}{227, 168, 105} 

\theoremstyle{plain}
\newtheorem{theorem}{Theorem}[section]
\newtheorem{proposition}[theorem]{Proposition}
\newtheorem{lemma}[theorem]{Lemma}
\newtheorem{corollary}[theorem]{Corollary}
\theoremstyle{definition}
\newtheorem{definition}[theorem]{Definition}

\theoremstyle{remark}
\newtheorem{remark}[theorem]{Remark}
\newtheorem{example}[theorem]{Example}

\usepackage[textsize=tiny]{todonotes}

\icmltitlerunning{
	Graph Automorphism Group Equivariant Neural Networks
}

\begin{document}

\twocolumn[
\icmltitle{
	Graph Automorphism Group Equivariant Neural Networks
}




\begin{icmlauthorlist}
	\icmlauthor{Edward Pearce--Crump}{imperial}
	\icmlauthor{William J. Knottenbelt}{imperial}
\end{icmlauthorlist}

\icmlaffiliation{imperial}{Department of Computing, 
Imperial College London, United Kingdom}

\icmlcorrespondingauthor{Edward Pearce--Crump}{ep1011@ic.ac.uk}

\icmlkeywords{Machine Learning, ICML}

\vskip 0.3in
]



\printAffiliationsAndNotice{}  

\begin{abstract}
	Permutation equivariant neural networks are typically used to learn
	from data that lives on a graph.  However, for any graph $G$ that has
	$n$ vertices, using the symmetric group $S_n$ as its group of
	symmetries does not take into account the relations that exist between
	the vertices.  Given that the actual group of symmetries is the
	automorphism group $\Aut(G)$, we show how to construct neural networks
	that are equivariant to $\Aut(G)$ by obtaining a full characterisation
	of the learnable, linear, $\Aut(G)$-equivariant functions between
	layers that are some tensor power of $\mathbb{R}^{n}$.  In particular,
	we find a spanning set of matrices for these layer functions in the
	standard basis of $\mathbb{R}^{n}$.  This result has important
	consequences for learning from data whose group of symmetries is a
	finite group because a theorem by \citet{frucht} showed that any finite
	group is isomorphic to the automorphism group of a graph. 
\end{abstract}

\section{Introduction}

In many cases, the relationships that exist between certain entities can be
structured in the form of a graph.  These include the interactions between
people in a social network \cite{leskovec}, the bonds that connect atoms in a
molecule \cite{gilmer}, and the relationships between users and items in a
recommendation system \cite{konstas}.  
As a result, there is a high level of motivation to design neural network
architectures that can offer new insights into data with this structure.
A number of methods for learning
from graphs exist in the machine learning literature.  These include Graph
Attention Networks (GATs) \cite{velickovic2018graph}, Graph Convolutional
Networks (GCNs) \cite{kipf2017semisupervised}, and Graph Neural Networks (GNNs)
\cite{gori, scarselli}, among others.  In particular, many of these approaches
use some form of equivariance to the symmetric group, which has since been 
characterised fully
by a number of authors \cite{maron2018, ravanbakhsh, pearcecrump}.
However, using neural networks that are equivariant to the symmetric group 
does not take into account the relations that exist between the vertices in a graph. 
Given that the actual group of symmetries of a graph is its automorphism group, 
we would like to construct neural networks that satisfy the stronger condition 
of being equivariant to this group instead.

In this paper, writing the automorphism group of a graph $G$ having some $n$
vertices as $\Aut(G)$, we give a full characterisation of all of the possible
$\Aut(G)$-equivariant neural networks whose layers are some tensor power of
$\mathbb{R}^{n}$ by finding a spanning set of matrices for the learnable,
linear, $\Aut(G)$-equivariant layer functions between such tensor power spaces
in the standard basis of $\mathbb{R}^{n}$.  Our approach is similar to the one
seen in \citet{pearcecrump, pearcecrumpB, pearcecrumpJ}, where they used the
combinatorics of set partitions to characterise the learnable, linear, group
equivariant layer functions between any two tensor power spaces of
$\mathbb{R}^{n}$ for a number of important groups.  However, instead of
calculating the spanning set by studying set partitions, we obtain it by
relating each spanning set element 
with the isomorphism class of
a so-called \textit{bilabelled graph}.
In short, a $(k,l)$--bilabelled graph is a graph that comes with two tuples,
one of length $k$ and the other of length $l$, whose entries are taken from the
vertex set of the graph (with repetitions amongst entries allowed).
Consequently, by looking at the combinatorics of bilabelled graphs, we can
determine the learnable, linear, $\Aut(G)$-equivariant layer
functions 
between any two tensor power spaces of $\mathbb{R}^{n}$.

We leverage the work of \citet{mancinska}, who studied the problem of
determining under what circumstances two graphs are quantum isomorphic.  They
built upon the work of \citet{chassaniol}, who showed that a vertex-transitive
graph has no quantum symmetries.  
We show how their results and methods can be
applied instead for the purpose of learning from data that has an underlying
symmetry to the automorphism group of a graph.

There are important consequences for performing such a characterisation.  A
famous theorem in algebraic graph theory known as Frucht's Theorem
\cite{frucht} states that every finite group is isomorphic to the automorphism
group of a finite undirected graph.  As a result, for any finite group of our
choosing, if we know the graph (having some $n$ vertices) whose automorphism
group is isomorphic to the group in question, then we will be able to
characterise all of the learnable, linear, group equivariant layer functions
between any two tensor power spaces of $\mathbb{R}^{n}$ for that group.  In
particular, we show that we can recover the diagram basis that appears in
\citet{godfrey} for the learnable, linear, $S_n$-equivariant layer functions
between tensor power spaces of $\mathbb{R}^{n}$ in the standard basis of
$\mathbb{R}^{n}$.  We also show that we can determine characterisations for
other groups, such as for $D_4$, considered as a subgroup of $S_4$, which, to
the best of our knowledge, have been missing from the literature.

Furthermore, even if, for a given finite undirected graph, we are unable to
determine the finite group that is isomorphic to the automorphism group of the
graph, we know that we can calculate the learnable, linear, automorphism group
equivariant linear layers using this method, which will be sufficient for
performing learning in situations where we do not need to know what the actual
automorphism group is explicitly.

The main contributions of this paper, which appear in Section
\ref{AutoGraphChar} onwards, are as follows: 
\begin{enumerate} 
	\item
		We are the first to show how the combinatorics
		underlying bilabelled graphs
		provides the theoretical foundation for
		constructing neural networks that are
		equivariant to the automorphism group of a
		graph $G$ having $n$ vertices where the layers
		are some tensor power of $\mathbb{R}^{n}$.  

	\item 
		In particular, we find a spanning set for the learnable,
		linear, $\Aut(G)$-equivariant layer functions between such
		tensor power spaces in the standard basis of $\mathbb{R}^{n}$.
	\item 
		We show how our approach can be used to recover the diagram basis
		that appears in \citet{godfrey}
		for the learnable, linear, $S_n$-equivariant layer functions 
		between tensor power spaces of $\mathbb{R}^{n}$.
\end{enumerate}

\textit{Notation}:
We let $[n]$
represent the set $\{1, \dots, n\}$ throughout this paper.

\section{Graph Theory Essentials}

We begin by recalling some of the fundamentals of graph theory.
For more details, see any standard book on
graph theory, such as \citet{bollobas}. 

\begin{definition}
	A \textbf{graph} $G$ is a tuple $(V(G), E(G))$ of sets,
	where $V(G)$ is a set of vertices for $G$
	and $E(G)$ is a subset of unordered pairs of elements 
	from $V(G) \times V(G)$ denoting the 
	undirected edges between the vertices of $G$.

	We include the possibility that the graph has loops; 
	however, we only allow at most one loop per vertex.
\end{definition}

\begin{definition}
	Let $G$ be a graph having $n$ vertices.  The \textbf{adjacency matrix}
	of $G$, denoted by $A_G$, is the $n \times n$ matrix whose
	$(i,j)$--entry is $1$ if vertex $i$ is adjacent to vertex $j$ in $G$,
	and is $0$ otherwise.  Note that $A_G$ is a symmetric matrix because
	$G$ has undirected edges, and the $(i,i)$-entry is $1$ in $A_G$ if and
	only if $G$ has a loop at vertex $i$.
\end{definition}

\begin{definition}
	Let $G$ be a graph.  The \textbf{complement} of $G$, denoted by
	$\overline{G}$, is the graph having the same vertex set as $G$ and the
	same loops as $G$, but now distinct vertices are adjacent in
	$\overline{G}$ if and only if they are not adjacent in $G$.
\end{definition}

\begin{example}
	The complete graph on $n$ vertices, $K_n$, is the loopless graph where
	every vertex is adjacent to every other vertex.
\end{example}

\begin{example}
	The cycle graph on $n$ vertices, $C_n$, is the loopless graph where
	every vertex $i \in [n]$ is adjacent to $j = i \pm 1 \mod n$, where $j
	\in [n]$.
\end{example}

\begin{definition} 
	Let $H$ and $G$ be graphs.  A \textbf{graph homomorphism} from $H$ to
	$G$ is a function $\phi: V(H) \rightarrow V(G)$ such that if $i$ is
	adjacent to $j$ in $H$, then $\phi(i)$ is adjacent to $\phi(j)$ in $G$.  
\end{definition}

\begin{definition} 
	Let $H$ and $G$ be graphs.  A \textbf{graph isomorphism} from $H$ to
	$G$ is a graph homomorphism that is also a bijection.
\end{definition}

Consequently, we get that
\begin{definition} 
	Let $G$ be a graph.  An \textbf{automorphism} of $G$ is an isomorphism
	from $G$ to $G$.  The set of all automorphisms of $G$, written
	$\Aut(G)$, can be shown to be a group under composition of functions. 
\end{definition}

\begin{remark}
	If $G$ is a graph having $n$ vertices, then $\Aut(G)$ is, in fact, a
	subgroup of $S_n$.  Specifically, if $\sigma \in S_n$, then it is easy
	to show, viewing $S_n$ as a subgroup of $GL(n)$, that
	\begin{equation}
		\sigma \in \Aut(G) \iff \sigma A_G = A_G \sigma
	\end{equation}
	It is also clear to see that $\Aut(G) \cong \Aut(\overline{G})$.
\end{remark}

\begin{example} \label{autocompletegraph}
	The automorphism group of the complete graph on $n$ vertices,
	$\Aut(K_n)$, is the symmetric group $S_n$.  Consequently, the
	automorphism group of the edgeless graph having $n$ vertices,
	$\Aut(\overline{K_n})$, is also the symmetric group $S_n$.
\end{example}

\begin{example} \label{autocyclegraph}
	The automorphism group of the cycle graph on $n$ vertices, $\Aut(C_n)$, 
	is isomorphic to the dihedral group $D_n$ of order $2n$.
\end{example}

\begin{example} \label{autotwocompletetwo}
	The automorphism group of two copies of the complete graph on two
	vertices, $\Aut(2K_2)$, is isomorphic to the dihedral group $D_4$ of
	order $8$.
\end{example}

\section{Graph Automorphism Group Equivariant Linear Layer Functions}

Neural networks that are equivariant to the 
the automorphism group of a graph $G$ 
having $n$ vertices, $\Aut(G)$,
can be constructed
by alternately composing
linear and non-linear equivariant functions between layer spaces that are a 
tensor power of $\mathbb{R}^{n}$ \cite{lim}. 
These layer spaces are representations of 
$\Aut(G)$ in the following sense.

Recall first that any $k$-tensor power of $\mathbb{R}^{n}$,
$(\mathbb{R}^{n})^{\otimes k}$,
is a representation of the symmetric group $S_n$, since the elements 
\begin{equation} \label{tensorelementfirst} 
	e_I \coloneqq e_{i_1} \otimes e_{i_2} \otimes \dots \otimes e_{i_k} 
\end{equation} 
for all $I \coloneqq (i_1, i_2, \dots, i_k) \in [n]^k$ 
form a basis of $(\mathbb{R}^{n})^{\otimes k}$, and the
action of $S_n$ that maps a basis element of $(\mathbb{R}^{n})^{\otimes k}$ of
the form (\ref{tensorelementfirst}) to 
\begin{equation} 
	e_{\sigma(I)} \coloneqq 
	e_{\sigma(i_1)} \otimes e_{\sigma(i_2)} \otimes \dots \otimes e_{\sigma(i_k)}
\end{equation} 
can be extended linearly on the basis.

For any graph $G$ having $n$ vertices, as $\Aut(G)$ is a subgroup of
$S_n$, we have that $(\mathbb{R}^{n})^{\otimes k}$ is also a
representation of $\Aut(G)$ that is given by the restriction of the
representation of $S_n$ to $\Aut(G)$.

We denote the representation of $S_n$ by $\rho_{k}$.  We will use the
same notation for the restriction of this representation to $\Aut(G)$,
with the context making clear that it is the restriction of the $S_n$
representation.

Moreover, we have that
\begin{definition} \label{equivariance}
	A map $\phi : 
	(\mathbb{R}^{n})^{\otimes k} \rightarrow
	(\mathbb{R}^{n})^{\otimes l}$
	is said to be \textbf{equivariant} to $\Aut(G)$ if,
	for all $\sigma \in \Aut(G)$ and $v \in 
	(\mathbb{R}^{n})^{\otimes k}$,
	\begin{equation} \label{equivmapdefn}
		\phi(\rho_{k}(\sigma)[v]) = \rho_{l}(\sigma)[\phi(v)]
	\end{equation}
	We denote the set of all \textit{linear} 
	$\Aut(G)$-equivariant maps 
	between $(\mathbb{R}^{n})^{\otimes k}$ and $(\mathbb{R}^{n})^{\otimes l}$ by
	\begin{equation} \label{desiredhomspace}
		\Hom_{\Aut(G)}((\mathbb{R}^{n})^{\otimes k},(\mathbb{R}^{n})^{\otimes l})
	\end{equation}
	It can be shown that 
	(\ref{desiredhomspace})
	is a vector space over $\mathbb{R}$. 
	See \citet{segal} for more details.
	Note that 
	(\ref{desiredhomspace})
	is a subspace of 
	$\Hom((\mathbb{R}^{n})^{\otimes k}, (\mathbb{R}^{n})^{\otimes l})$, 
	the vector space of all linear maps from 
	$(\mathbb{R}^{n})^{\otimes k}$ 
	to $(\mathbb{R}^{n})^{\otimes l}$.
\end{definition}

Our goal is 
to calculate all of the weight matrices
that can appear between any two layers of the 
graph automorphism group equivariant neural networks in question.
It is enough to 
construct a spanning set of matrices for 
$\Hom_{\Aut(G)}((\mathbb{R}^{n})^{\otimes k}, (\mathbb{R}^{n})^{\otimes l})$,
by viewing it as a subspace of 
$\Hom((\mathbb{R}^{n})^{\otimes k}, (\mathbb{R}^{n})^{\otimes l})$
and choosing the standard basis of
$\mathbb{R}^{n}$,
since any weight matrix will be a 
weighted linear combination of these spanning set matrices.


\section{Characterisation Result using Bilabelled Graphs}
\label{AutoGraphChar}


In this section, we construct a spanning set of matrices for
$\Hom_{\Aut(G)}((\mathbb{R}^{n})^{\otimes k}, (\mathbb{R}^{n})^{\otimes l})$
by relating each spanning set matrix with 
the isomorphism class of
a so-called $(k,l)$--bilabelled graph.
We leverage the work of
\citet{mancinska} throughout, but begin by stating a result found by
\citet{chassaniol} which describes,
in terms of a generating set of matrices,
the category whose morphisms are the linear layer functions that we wish to
characterise.

\subsection{Chassaniol's Result}

For each group $G$ that is a subgroup of $S_n$, we can define the
following category.  
\begin{definition} \label{catgroupreps} 
	The category $\mathcal{C}(G)$ consists of objects that are the
	$k$-order tensor power spaces of $\mathbb{R}^{n}$, as
	representations of $G$, and morphism spaces between any two
	objects that are the vector spaces
	$\Hom_{G}((\mathbb{R}^{n})^{\otimes k},
	(\mathbb{R}^{n})^{\otimes l})$.

	The vertical composition of morphisms is given by the usual
	composition of linear maps, the tensor product is given by the
	usual tensor product of linear maps, and the unit object 
	is the one-dimensional trivial representation of $G$.  
\end{definition}

\begin{remark} 
	We will sometimes write $\mathcal{C}(G)(k,l)$ for
	$\Hom_{G}((\mathbb{R}^{n})^{\otimes k}, (\mathbb{R}^{n})^{\otimes l})$,
	and reuse the notation $\mathcal{C}(G)$ for the set $\cup_{k,l =
	0}^{\infty} \; \mathcal{C}(G)(k,l)$.  
\end{remark} 

\begin{proposition}
	\label{tenscatduals}
	The category $\mathcal{C}(G)$ is a 
	strict $\mathbb{R}$-linear monoidal category.
\end{proposition}

\begin{proof}
	See the Technical Appendix.	
\end{proof}



It will be useful to define a number of spider maps $M^{k,l}$
together with the swap map $S$.
\begin{definition} 
	\label{spiderdefn}
	For all non-negative integers $k, l$,
	the \textbf{spider map} $M^{k,l} \in
	\Hom((\mathbb{R}^{n})^{\otimes k}, (\mathbb{R}^{n})^{\otimes
	l})$ is defined as follows: 
	\begin{itemize} 
		\item if $k, l > 0$, then $M^{k,l}$ maps $e_i^{\otimes k}$ to
		$e_i^{\otimes l}$ for all $i \in [n]$ and maps all
	other vectors to the zero vector.
		\item if $k = 0$
		and $l > 0$, then $M^{0,l}$ maps $1$ to $\sum_i
	e_i^{\otimes l}$. 
		\item if $k > 0$ and $l = 0$, then $M^{k,0}$ maps $e_i^{\otimes k}$ to $1$ for
			all $i \in [n]$ and maps all other vectors to
			$0$.
		\item if $k = 0$
	and $l = 0$, then $M^{0,0} \coloneqq (n).$ \end{itemize}

	The \textbf{swap map} $S \in
	\Hom((\mathbb{R}^{n})^{\otimes 2}, (\mathbb{R}^{n})^{\otimes
	2})$ maps $e_i \otimes e_j$ to $e_j \otimes e_i$ for all $i, j \in [n]$.
\end{definition}


For any graph $G$ having $n$ vertices, \citet{chassaniol}
found the following
generating set for the category $\mathcal{C}(\Aut(G))$:

\begin{theorem}[\citet{chassaniol}, Proposition 3.5] \label{chassaniol}
	Let $A_G$ denote the adjacency matrix of the graph $G$. Then 
	\begin{equation}
		\mathcal{C}(\Aut(G)) = 
		\langle M^{0,1}, M^{2,1}, A_G, S \rangle_{+, \circ, \otimes, *} 
	\end{equation} 
	where the right hand side denotes all matrices that can be generated
	from the four matrices using the operations of $\mathbb{R}$-linear
	combinations, matrix product, Kronecker product, and transposition.  
\end{theorem}

\subsection{The Category of All Bilabelled Graphs}

\citet{mancinska} showed how Chassaniol's generating set for
$\mathcal{C}(\Aut(G))$ can be improved by relating it to the combinatorics of
bilabelled graphs.

\begin{figure}[t] 
	\begin{center}
		\scalebox{0.8}{\input{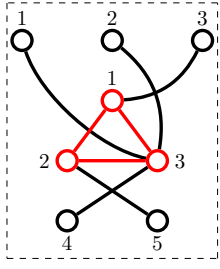}} 
		\caption{The $(2,3)$--bilabelled graph diagram 
		that is associated with the isomorphism class of
		$\bm{K} \coloneqq (K_3, (3,2), (3,3,1))$.}
	\label{bilabelledgraphex} 
	\end{center} 
\end{figure}

\begin{definition}[Bilabelled Graph] 
	A $\bm{(k,l)}$\textbf{--bilabelled graph} $\bm{H}$ is a triple $(H, \bm{k}, \bm{l})$,
	where $H$ is a graph having some $m$ vertices with labels in $[m]$,
	$\bm{k} \coloneqq (k_1, \dots, k_k)$ is a tuple in $[m]^k$, and $\bm{l}
	\coloneqq (l_1, \dots, l_l)$ is a tuple in $[m]^l$.

	We call $\bm{k}$ and $\bm{l}$ the \textbf{input} and \textbf{output
	tuples} to $\bm{H}$ respectively, and we call $H$ the
	\textbf{underlying graph} of $\bm{H}$.  A vertex in the underlying
	graph $H$ of $\bm{H}$ is said to be \textbf{free} if it does not appear in
	either of the input or output tuples. 
\end{definition}

In order to provide a diagrammatic representation of bilabelled graphs, we need
the following definition.
\begin{definition}
	Let $\bm{H_1} = (H_1, \bm{k}, \bm{l})$ be a $(k,l)$--bilabelled
	graph and let $\bm{H_2} = (H_2, \bm{k'}, \bm{l'})$ be another
	$(k,l)$--bilabelled graph.  Then $\bm{H_1}$ and $\bm{H_2}$ are
	\textbf{isomorphic} as $(k,l)$--bilabelled graphs, written $\bm{H_1} \cong
	\bm{H_2}$, if and only if there is a graph isomorphism from $H_1$ to
	$H_2$ such that $k_i \mapsto k_i'$ for all $i \in [k]$ and $l_j \mapsto
	l_j'$ for all $j \in [l]$.
	We denote $\bm{H_1}$'s \textbf{isomorphism class} by $[\bm{H_1}]$.
\end{definition}

\begin{remark}
	An isomorphism between two $(k,l)$--bilabelled graphs can effectively be
	thought of as a relabelling of the vertices of the \textit{same}
	underlying graph, resulting in an appropriate relabelling of the
	elements of the tuples themselves.  Consequently, the isomorphism can
	be thought of as a relabelling of the same bilabelled graph.
\end{remark}

With this in mind, we can represent the isomorphism class
$[\bm{H}]$ for a $(k,l)$--bilabelled graph $\bm{H} = (H, \bm{k},
\bm{l})$ in diagrammatic form.  We choose $\bm{H}$ as a class
representative of $[\bm{H}]$, and proceed as follows: we draw
\begin{itemize} 
	\item 
		$l$ black vertices on the top row, labelled left to right by
		$1, \dots, l$, 
	\item 
		$k$ black vertices on the bottom row, labelled left to right
		by $l+1, \dots, l+k$, 
	\item 
		the underlying graph $H$ in red, with labelled red vertices and
		red edges, in between the two rows of black vertices, and 
	\item 
		a black line connecting vertex $i$ in the top row to $l_i$, and
		a black line connecting vertex $l + j$ in the bottom row to $k_j$.  
\end{itemize} 
Note that if we drew a diagram for another representative of the same
isomorphism class $[\bm{H}]$, by the definition of an isomorphism for
$(k,l)$--bilabelled graphs, we would obtain exactly the same diagram as for the
first class representative, except the labels of the underlying graph's
vertices (and consequently the elements of the tuples) would be different.  As
a result, the diagram for the isomorphism class $[\bm{H}]$ is independent of
the choice of class representative, and so we choose the class' label,
$\bm{H}$, to draw the diagram throughout, unless otherwise stated.  With the
technicalities being understood, we often refer to the construction defined
above as a $(k,l)$--bilabelled graph diagram \textit{for} $\bm{H}$
\textit{itself}, and use the same notation $\bm{H}$.  We give an example of a
$(2,3)$--bilabelled graph diagram in Figure \ref{bilabelledgraphex}.

We can define a number of \textbf{operations} on isomorphism classes of bilabelled
graphs, namely, composition, tensor product, and involution, as follows.

\begin{definition}[Composition of Bilabelled Graphs] \label{compositiongraph}
	Let $[\bm{H_1}]$ be the isomorphism class of the $(k,l)$--bilabelled
	graph $\bm{H_1} = (H_1, \bm{k}, \bm{l})$, and let $[\bm{H_2}]$ be the
	isomorphism class of the $(l,m)$--bilabelled graph $\bm{H_2} = (H_2,
	\bm{l'}, \bm{m})$.

	Then we define the \textbf{composition} $[\bm{H_2}] \circ [\bm{H_1}]$ to be the
	isomorphism class $[\bm{H}]$ of the $(k,m)$--bilabelled graph $\bm{H} =
	(H, \bm{k}, \bm{m})$ that is obtained as follows: drawing each
	isomorphism class as a diagram, we first relabel the red vertices in
	$H_2$ (and consequently the elements of $\bm{l'}, \bm{m}$) under the
	map $i \mapsto i'$, for all $i \in [V(H_2)]$.  Then, we connect the top
	row of black vertices of $\bm{H_1}$ with the bottom row of black
	vertices of $\bm{H_2}$, and delete the vertices themselves.  We are now
	left with $l$ black lines that are edges between red vertices of the
	underlying graphs $H_1$ and $H_2$.  Next, we contract these, forming
	set unions of the vertex labels where appropriate, and remove any red
	multiedges between red vertices that appear in the contraction to
	obtain the new underlying graph $H$.  Note that we keep any loops that
	appear in the new underlying graph $H$.  Finally, we relabel the vertex
	set of the new underlying graph $H$ so that each vertex is labelled by
	an integer only, and consequently relabel the entries of the tuples
	$\bm{k}$ and $\bm{m}$ accordingly.

	Since this operation has been defined on diagrams, it means that the
	operation is well defined on the isomorphism classes themselves.
	We give an example of this composition in Figure \ref{graphcomp}.
\end{definition}

\begin{remark} 
	In order to compose two isomorphism classes of bilabelled graphs, note
	that only the number of bottom row vertices in the diagram for
	$\bm{H_2}$ needs to be equal to the number of top row vertices in the
	diagram for $\bm{H_1}$.  In particular, the number of vertices in the
	underlying graphs of $\bm{H_1}$ and $\bm{H_2}$ do \textit{not} need to
	be the same.
\end{remark}

\begin{definition}[Tensor Product of Bilabelled Graphs] \label{tensorprodgraph}
	Let $[\bm{H_1}]$ be the isomorphism class of the $(k,l)$--bilabelled
	graph $\bm{H_1} = (H_1, \bm{k}, \bm{l})$ and let $[\bm{H_2}]$ be the
	isomorphism class of the $(q,m)$--bilabelled graph $\bm{H_2} = (H_2,
	\bm{q}, \bm{m})$.

	Then we define the \textbf{tensor product} $[\bm{H_1}] \otimes [\bm{H_2}]$ to be
	the isomorphism class of the $(k+q,l+m)$--bilabelled graph $(H_1 \cup
	H_2, \bm{k}\bm{q}, \bm{l}\bm{m})$ where $\bm{k}\bm{q}$ is the
	$(k+q)$--length tuple obtained by concatenating $\bm{k}$ and $\bm{q}$,
	and likewise for $\bm{l}\bm{m}$.  
\end{definition}

\begin{definition}[Involution of Bilabelled Graphs] \label{involutiongraph}
	Let $[\bm{H}]$ be the isomorphism class of the $(k,l)$--bilabelled
	graph $\bm{H} = (H, \bm{k}, \bm{l})$.  Then we define the \textbf{involution}
	$[\bm{H^{*}}]$ to be the isomorphism class of the $(l,k)$--bilabelled
	graph $(H, \bm{l}, \bm{k})$.
\end{definition}

\begin{figure}[t] 
	\begin{center}
	\scalebox{0.8}{\input{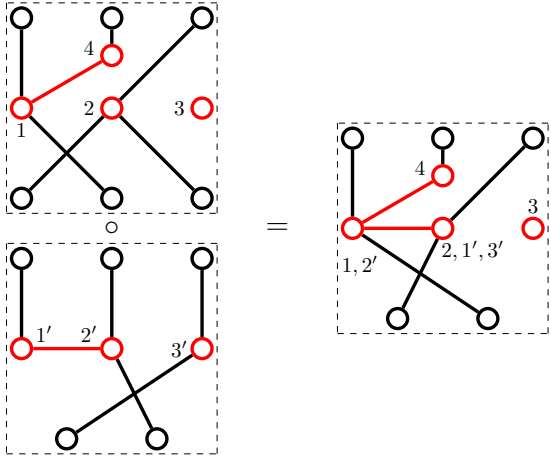}}
	\caption{The composition $[\bm{H_2}] \circ [\bm{H_1}]$ of the
	$(2,3)$--bilabelled graph diagram for $\bm{H_1} = (H_1, (3',2'),
	(1',2',3'))$ with the $(3,3)$--bilabelled graph diagram for $\bm{H_2} =
	(H_2, (2,1,2), (1,4,2))$, where $H_1$ is the graph having (relabelled)
	vertex set $[3']$ and edge set $\{(1',2')\}$ and $H_2$ is the graph
	having vertex set $[4]$ and edge set $\{(1,4)\}$.  The vertices of the
	resulting bilabelled graph diagram on the RHS would be relabelled. For
	example, $\{1,2'\}$ could be relabelled as $1$ and $\{2,1',3'\}$
	could be relabelled as $2$ to give the $(2,3)$--bilabelled graph
	diagram for $\bm{H} = (H, (2,1), (1,4,2))$, where $H$ is the graph having
	vertex set $[4]$ and edge set $\{(1,4), (1,2)\}$.}
	\label{graphcomp}
	\end{center} 
\end{figure}

We can form a category for the bilabelled graphs, as follows:

\begin{definition}[Category of Bilabelled Graphs] 
	The \textbf{category of all bilabelled graphs} $\mathcal{G}$ is the category
	whose objects are the non-negative integers,
	and, for any pair of objects $k$ and $l$, the morphism
	space $\mathcal{G}(k,l) \coloneqq \Hom_{\mathcal{G}}(k,l)$ is defined
	to be the $\mathbb{R}$-linear span of the set of all isomorphism
	classes of $(k,l)$--bilabelled graphs.

	The vertical composition of morphisms is the composition of isomorphism
	classes of bilabelled graphs given in Definition~\ref{compositiongraph} that is extended to be $\mathbb{R}$-bilinear, 
	the tensor product of morphisms is the tensor
	product of isomorphism classes of bilabelled graphs given in Definition
	\ref{tensorprodgraph} that is also extended to be $\mathbb{R}$-bilinear, 
	and the unit object is $0$.
\end{definition}

\begin{proposition} \label{catbilgraph}
	The category of all bilabelled graphs, $\mathcal{G}$,
	is a strict $\mathbb{R}$--linear monoidal category.
\end{proposition}

\begin{proof}
	See the Technical Appendix.
\end{proof}

\subsection{$G$-Homomorphism Matrices}

\citet{mancinska} established a relationship between the abstract and the
concrete: namely, between isomorphism classes of $(k,l)$--bilabelled graphs
and, for a fixed graph $G$ having $n$ vertices, matrices that are linear maps
$(\mathbb{R}^{n})^{\otimes k} \rightarrow (\mathbb{R}^{n})^{\otimes l}$, which
they termed $G$-homomorphism matrices.  We express this relationship more
formally in terms of functors and categories at the end of this section.

\begin{definition}[$G$-Homomorphism Matrix] \label{GHomMatrix}
	Suppose that $G$ is a graph having $n$ vertices, and let $[\bm{H}]$ be
	the isomorphism class of the $(k,l)$--bilabelled graph $\bm{H}
	\coloneqq (H, \bm{k}, \bm{l})$.


	The $G$\textbf{-homomorphism matrix} of $[\bm{H}]$ is the $n^l \times n^k$ matrix
	where each $(I,J)$-entry is given by the number of graph
	homomorphisms from $H$ to $G$ such that $\bm{l}$ is mapped to $I$ and
	$\bm{k}$ is mapped to $J$. 
	We denote this matrix by $X_{\bm{H}}^{G}$.
\end{definition}

\begin{remark}
	Note that a $G$-homomorphism matrix $X_{\bm{H}}^{G}$ 
	is independent of the 
	choice of class representative for $[\bm{H}]$, 
	and so we can refer to a $G$-homomorphism matrix 
	\textit{of} $\bm{H}$ \textit{itself},
	with the technicalities being understood.
	Also, any such matrix must have only real entries, by definition.
\end{remark}

In Figure \ref{AutGIJdiagram},
we present an example that shows how to calculate the 
$G$-homomorphism matrix of $[\bm{H}]$ for the graph 
$G$ having vertex set $V(G) = \{1, 2, 3\}$ and edge set
$E(G) = \{(1,2)\}$ and for the $(1,1)$--bilabelled graph
$\bm{H} = (H, (3), (1))$, where $H$ is the graph having vertex set
$V(H) = \{1, 2, 3, 4\}$ and edge set $E(H) = \{(1,2), (3,4)\}$.
To determine the $(I,J)$-entry of $X_{\bm{H}}^{G}$, we superimpose
the values of $I$ onto the top row of black vertices in $\bm{H}$,
and the values of $J$ onto the bottom row of black vertices in $\bm{H}$.
For each red vertex that is connected with a set of black vertices,
we look to update its label. 
If all of the black vertices in the set have the same label, then
we update the red vertex with that label, otherwise we stop and immediately 
determine that the $(I,J)$-entry of $X_{\bm{H}}^{G}$ is $0$.
Assuming that these red vertices can and have been updated, we
determine the number of possible mappings to $G$ for
the red vertices that have not been 
updated in $\bm{H}$ 
such that the overall mapping 
is a graph homomorphism.
The total number of possibilities is the $(I,J)$-entry of $X_{\bm{H}}^{G}$.

We now describe some important examples of $G$-homomorphism matrices
where $G$ is a graph having $n$ vertices throughout.

\begin{example} \label{adjacencyex}
	If $\bm{A}$ is the $(1,1)$--bilabelled graph 
	$(K_2, (1), (2))$, where $K_2$
	is the complete graph on two vertices, then 
	$X_{\bm{A}}^{G}$ is the adjacency matrix $A_G$ of $G$.
\end{example}

\begin{example} \label{spidergraphex}
	If $\bm{M}^{k,l}$ is the $(k,l)$--bilabelled graph 
	$(K_1, \bm{k} = (1, \dots, 1), \bm{l} = (1, \dots, 1))$, where $K_1$
	is the complete graph on one vertex, then 
	$X_{\bm{M}^{k,l}}^{G}$ is the spider matrix 
	$M^{k,l}$ that is given in Definition \ref{spiderdefn}.
\end{example}
			
\begin{example} \label{swapex}
	If $\bm{S}$ is the $(2,2)$--bilabelled graph $(\overline{K_2}, (2, 1),
	(1, 2))$, where $\overline{K_2}$ is the edgeless graph on two vertices,
	then $X_{\bm{S}}^{G}$ is the swap map $S$ 
	that is also given in Definition \ref{spiderdefn}.
\end{example}



\begin{figure*}[tb]
	\begin{tcolorbox}[colback=white!02, colframe=black]
	\begin{center}
		\scalebox{0.68}{\input{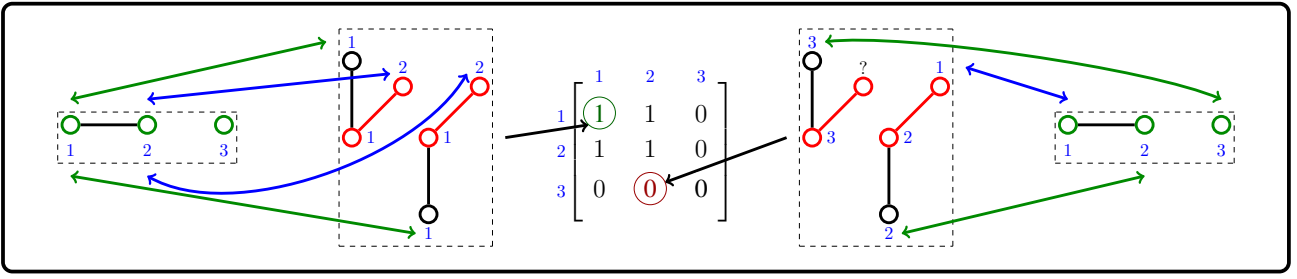}}
	\end{center}
	\end{tcolorbox}
	\caption{
		For the graph $G$ having vertex set $V(G) = \{1, 2, 3\}$ and
		edge set $E(G) = \{(1,2)\}$, we show how to calculate the
		$(I,J)$-entries of the $G$-homomorphism matrix $X_{\bm{H}}^{G}$
		corresponding to the isomorphism class $[\bm{H}]$ of the
		$(1,1)$--bilabelled graph $\bm{H} = (H, (3), (1))$ where $H$ is
		the graph having vertex set $V(H) = \{1, 2, 3, 4\}$ and edge
		set $E(H) = \{(1,2), (3,4)\}$.  On the left hand side, we
		calculate the $(1,1)$-entry of $X_{\bm{H}}^{G}$.  Since both of
		the black labelled vertices are being mapped to vertex $1$ in
		$G$, we can superimpose $1$ onto these vertices, and hence also
		onto the red labelled vertices that are connected with them, 
		to determine where the other
		red vertices can be mapped to under a graph homomorphism. In
		this case, the only possible vertex in $G$ that the other red
		vertices can be mapped to is $2$.  Hence there is only one
		possible graph homomorphism from $H$ to $G$ such that $1
		\mapsto 1$ and $3 \mapsto 1$, and so the $(1,1)$-entry of
		$X_{\bm{H}}^{G}$ is $1$.  On the right hand side, we calculate
		the $(3,2)$-entry of $X_{\bm{H}}^{G}$.  We follow the same
		approach by superimposing $3$ in $G$ onto the black vertex
		labelled $1$ in $\bm{H}$, and $2$ in $G$ onto the black vertex
		labelled $3$ in $\bm{H}$.  We relabel the red vertices that the
		black vertices are connected with and see where the other red
		vertices can be mapped to under a graph homomorphism.  
		Whilst 
		one of these red vertices can only be mapped to $1$ in $G$,
		the other red vertex
		cannot be mapped to any vertex in $G$,
		since the vertex $3$ in $G$ is not connected with any other
		vertex in $G$.  Hence the number of graph homomorphisms from
		$H$ to $G$ such that $1 \mapsto 3$ and $3 \mapsto 2$ is $0$,
		and so the $(3,2)$-entry of $X_{\bm{H}}^{G}$ is $0$.
		}
	\label{AutGIJdiagram}
\end{figure*}


For a fixed graph $G$ having $n$ vertices, we can form a category
for the $G$-homomorphism matrices, as follows:

\begin{definition}[Category of $G$-Homomorphism Matrices] 
	For a given graph $G$ having $n$ vertices, \textbf{the category of all}
	$G$\textbf{-homomorphism matrices}, $\mathcal{C}^G$, is the category whose
	objects are the vector spaces $(\mathbb{R}^{n})^{\otimes k}$, and, for
	any pair of objects $(\mathbb{R}^{n})^{\otimes k}$ and
	$(\mathbb{R}^{n})^{\otimes l}$, the morphism space
	$\Hom_{\mathcal{C}^G}((\mathbb{R}^{n})^{\otimes k},
	(\mathbb{R}^{n})^{\otimes l})$ is defined to be the $\mathbb{R}$-linear
	span of the set of all $G$-homomorphism matrices obtained from all
	isomorphism classes of $(k,l)$--bilabelled graphs.

	The vertical composition of morphisms is given by the usual
	multiplication of matrices, the tensor product of morphisms is given by
	the usual Kronecker product of matrices, and the unit object is
	$\mathbb{R}$.
\end{definition}

\begin{proposition} \label{catGhommat}
	The category of all
	$G$-homomorphism matrices for a given graph $G$ having $n$ vertices,
	$\mathcal{C}^G$,
	is a strict $\mathbb{R}$--linear monoidal category.
\end{proposition}

\begin{proof}
	See the Technical Appendix.
\end{proof}

\citet{mancinska} showed that the operations given
on isomorphism classes of bilabelled graphs $\bm{H}$
correspond bijectively with the matrix operations on
$G$-homomorphism matrices $X_{\bm{H}}^G$.  This is
stated more formally in the following 
three lemmas:

\begin{lemma}[\citet{mancinska}, Lemma 3.21]
	Suppose that $G$ is a graph having $n$ vertices.  Let $[\bm{H_1}]$ be
	the isomorphism class of the $(k,l)$--bilabelled graph $\bm{H_1} =
	(H_1, \bm{k}, \bm{l})$, and let $[\bm{H_2}]$ be the isomorphism class
	of the $(l,m)$--bilabelled graph $\bm{H_2} = (H_2, \bm{l'}, \bm{m})$.
	Then
	\begin{equation}
		X_{\bm{H_2}}^G X_{\bm{H_1}}^G 
		=
		X_{\bm{H_2 \circ H_1}}^G
	\end{equation}
\end{lemma}

\begin{lemma} [\citet{mancinska}, Lemma 3.23] \label{gHomtensorprod}
	Suppose that $G$ is a graph having $n$ vertices.  Let $[\bm{H_1}]$ be
	the isomorphism class of the $(k,l)$--bilabelled graph $\bm{H_1} =
	(H_1, \bm{k}, \bm{l})$, and let $[\bm{H_2}]$ be the isomorphism class
	of the $(q,m)$--bilabelled graph $\bm{H_2} = (H_2, \bm{q}, \bm{m})$.
	Then
	\begin{equation}
		X_{\bm{H_1}}^G \otimes X_{\bm{H_2}}^G 
		=
		X_{\bm{H_1 \otimes H_2}}^G
	\end{equation}
\end{lemma}

\begin{lemma}[\citet{mancinska}, Lemma 3.24]
	Suppose that $G$ is a graph having $n$ vertices, and let $[\bm{H}]$ be
	the isomorphism class of the $(k,l)$--bilabelled graph $\bm{H} = (H,
	\bm{k}, \bm{l})$.
	Then
	\begin{equation}
		(X_{\bm{H}}^G)^{*}
		=
		X_{\bm{H^{*}}}^G
	\end{equation}
	where $(X_{\bm{H}}^G)^{*}$ is the transpose of the 
	matrix $X_{\bm{H}}^G$.
\end{lemma}

Consequently, we obtain the following theorem, 
expressed in terms of functors and categories.

\begin{theorem} \label{graphfunctor}
	Suppose that $G$ is a graph having $n$ vertices.
	Then there exists a full, strict $\mathbb{R}$--linear monoidal functor
	\begin{equation}
		\mathcal{F}^G : \mathcal{G} \rightarrow \mathcal{C}^G
	\end{equation}
	that is defined on the objects of $\mathcal{G}$ by
	$\mathcal{F}^G(k) \coloneqq (\mathbb{R}^{n})^{\otimes k}$
	and, for any objects $k,l$ of $\mathcal{G}$, the map
	\begin{equation}
		\Hom_{\mathcal{G}}(k,l)
		\rightarrow
		\Hom_{\mathcal{C}^G}(\mathcal{F}^G(k),\mathcal{F}^G(l))
	\end{equation}
	is given by
	\begin{equation}
		[\bm{H}] \mapsto X_{\bm{H}}^G 
	\end{equation}
	for all isomorphism classes of $(k,l)$--bilabelled graphs. 
\end{theorem}

\begin{proof}
	See the Technical Appendix.
\end{proof}

\subsection{A Spanning Set of Matrices for
the Learnable, Linear, $\Aut(G)$-Equivariant Layer Functions}

Compare the following proposition with Chassaniol's
result, given in Theorem \ref{chassaniol}.

\begin{proposition}[\citet{mancinska}, Theorem 8.4] \label{graphcatgenset}
	We have that 
	\begin{equation} 
		\mathcal{G} = 
		\langle [\bm{M}^{0,1}], [\bm{M}^{2,1}], [\bm{A}], [\bm{S}] \rangle_{\circ, \otimes, *}
	\end{equation} 
	where $\bm{M}^{0,1}, \bm{M}^{2,1}, \bm{A}, \bm{S}$ are
	the bilabelled graphs defined in Examples \ref{adjacencyex},
	\ref{spidergraphex}, and \ref{swapex}, and the operations 
	$\circ, \otimes, *$ on bilabelled graphs are those that are given in Definitions
	\ref{compositiongraph}, \ref{tensorprodgraph} and
	\ref{involutiongraph}, respectively.
\end{proposition}

We have come to the main results of this paper, 
namely the following theorem and its corollary.

\begin{theorem} \label{graphmainthrm}
	Suppose that $G$ is a graph having $n$ vertices.  
	The vector space
	of all $\Aut(G)$-equivariant, linear maps between tensor power spaces
	of $\mathbb{R}^{n}$, $ \Hom_{\Aut(G)}((\mathbb{R}^{n})^{\otimes k},
	(\mathbb{R}^{n})^{\otimes l})$, when the standard basis of
	$\mathbb{R}^{n}$ is chosen, is spanned by all of the $G$-homomorphism matrices
	$X_{\bm{H}}^{G}$ that are obtained from all of the isomorphism classes of
	$(k,l)$--bilabelled graphs.
\end{theorem}

\begin{proof}
	The proof of this theorem is 
	inspired by 
	\citet{mancinska},
	Theorem 8.5. 

	We know, by Chassaniol, that
	\begin{equation}
		\mathcal{C}(\Aut(G)) = 
		\langle M^{0,1}, M^{2,1}, A_G, S \rangle_{+, \circ, \otimes, *} 
	\end{equation} 
	By Examples \ref{adjacencyex}, \ref{spidergraphex}, and \ref{swapex},
	we get that $\mathcal{C}(\Aut(G))$ is equal to
	\begin{equation}
		\{X_{\bm{H}}^{G} \mid [\bm{H}] \in \langle [\bm{M}^{0,1}], [\bm{M}^{2,1}], [\bm{A}], [\bm{S}] \rangle_{+, \circ, \otimes, *} \}
	\end{equation} 
	and so, by Proposition \ref{graphcatgenset}, we have that 
	\begin{equation}
		\mathcal{C}(\Aut(G)) = 
		\mathbb{R}\text{--span}\{X_{\bm{H}}^{G} \mid [\bm{H}] \in \mathcal{G}\}
	\end{equation} 
	Consequently, for any $k,l$, we get that
	\begin{equation} \label{autspanres}
		\mathcal{C}(\Aut(G))(k,l)
		= 
		\mathbb{R}\text{--span}\{X_{\bm{H}}^{G} \mid [\bm{H}] \in \mathcal{G}(k,l)\}
	\end{equation}
	As the LHS of (\ref{autspanres}) is equivalent notation for	
	$\Hom_{\Aut(G)}((\mathbb{R}^{n})^{\otimes k}, (\mathbb{R}^{n})^{\otimes l})$, 
	we obtain our result.
\end{proof}

\begin{corollary}
	For all non-negative integers $l$ and $k$, if $G$ is a graph having $n$ vertices,
	then the weight matrix $W$ that appears in an $\Aut(G)$-equivariant linear
	layer function from 
	$(\mathbb{R}^{n})^{\otimes k}$ to
	$(\mathbb{R}^{n})^{\otimes l}$
	must be of the form
	\begin{equation}
		W =
		\sum_{[\bm{H}] \in \mathcal{G}(k,l)} \lambda_{\bm{H}} X_{\bm{H}}^{G}
	\end{equation}
	for weights $\lambda_{\bm{H}} \in \mathbb{R}$.
\end{corollary}

\begin{remark}
	In particular, Theorem \ref{graphmainthrm} shows that 
	$\mathcal{C}^G$ 
	is isomorphic to
	$\mathcal{C}(\Aut(G))$ 
	as categories, 
	and that the objects of the category $\mathcal{C}^G$
	come from
	representations of $\Aut(G)$. 
\end{remark}

Theorem \ref{graphmainthrm} is especially powerful when it is combined
with Frucht's Theorem \cite{frucht}. Frucht's Theorem states
that every finite group is isomorphic to the automorphism group
of a finite undirected graph. 
Hence, for any finite group,
by determining a graph (having $n$ vertices) whose automorphism group is isomorphic 
to the group in question,
we can determine the weight matrix that appears between
any two layers that are some tensor power of $\mathbb{R}^{n}$ 
in a neural network that is equivariant to the group.

\begin{remark}
	It is very important to note the following.  If $H$ is a group that is
	isomorphic to the automorphism group of a graph $G$ having $n$
	vertices, then the spanning set that we obtain for
	$\Hom_{H}((\mathbb{R}^{n})^{\otimes k}, (\mathbb{R}^{n})^{\otimes l})$
	using Theorem \ref{graphmainthrm} 
	depends on how the automorphism group
	is embedded in the symmetric group $S_n$, itself thought of as a
	subgroup of matrices in $GL(n)$ having chosen the standard basis of
	$\mathbb{R}^{n}$.  The spanning set is determined not only by how the
	vertices of the underlying graph $G$ are labelled (up to all
	automorphisms), but also by what the edges are in $G$.
\end{remark}

	For example, in the Technical Appendix, we show not only that $D_4$ has
	three embeddings in $S_4$, but also that each embedding can be obtained
	separately from two graphs that are the complement of each other, where
	both graphs have the same labelling of the vertices.  All six instances
	(an embedding of $D_4$ coming from a graph with a certain labelling of
	its vertices) 
	give rise to different, albeit related, spanning sets (in
	fact, bases) of $\Hom_{D_4}(\mathbb{R}^{4}, \mathbb{R}^{4})$. 
	Hence,
	we obtain a basis for each specific embedding and labelled graph that
	is chosen.

We have seen that 
we can find a spanning set for
$\Hom_{\Aut(G)}((\mathbb{R}^{n})^{\otimes k}, (\mathbb{R}^{n})^{\otimes l})$
by considering all isomorphism classes
of $(k,l)$--bilabelled graphs.
However, 
with the following result,
we can reduce the number of elements
that appear in the spanning set by reducing the number of isomorphism classes
that we need to consider.

\begin{proposition} \label{disconngraph}
	Let $\bm{H}$ be a $(k,l)$--bilabelled graph whose underlying
	graph $H$ contains a subset of free vertices that, whilst they may have
	edges amongst themselves, are entirely disconnected from any vertices
	that are in a connected component containing a non-free vertex.
	
	Then $X_{\bm{H}}^{G}$ is a scalar multiple of $X_{\bm{H'}}^{G}$, where
	$\bm{H'}$ is the $(k,l)$--bilabelled graph that is 
	obtained from
	$\bm{H}$ by removing the subset.
\end{proposition}

\begin{proof}
	See the Technical Appendix.
\end{proof}

We also have the following result that is useful for generating
the isomorphism classes of $(k,l)$--bilabelled graphs.
\begin{proposition}[Frobenius Duality] \label{FrobDuality}
	Suppose that we use Theorem \ref{graphmainthrm} to obtain a spanning
	set for $\Hom_{\Aut(G)}((\mathbb{R}^{n})^{\otimes k},
	(\mathbb{R}^{n})^{\otimes l})$.  Then we can immediately obtain a
	spanning set for $\Hom_{\Aut(G)}((\mathbb{R}^{n})^{\otimes q},
	(\mathbb{R}^{n})^{\otimes m})$, for any $q, m \geq 0$ such that $q + m
	= k + l$.
\end{proposition}

\begin{proof}
	See the Technical Appendix.
\end{proof}

\begin{remark}
	In the Technical Appendix, we show that we can recover the diagram
	basis for $\Hom_{S_n}((\mathbb{R}^{n})^{\otimes k},
	(\mathbb{R}^{n})^{\otimes l})$ by looking at the homomorphism matrices for
	the complement of the complete graph on $n$ vertices, $\overline{K_n}$.
	
	This implies that, for a graph $G$ having $n$
	vertices, the spanning set of $\Hom_{\Aut(G)}((\mathbb{R}^{n})^{\otimes
	k}, (\mathbb{R}^{n})^{\otimes l})$
	automatically includes the image, under the functor $\mathcal{F}^G$, of all set
	partitions of $[l+k]$ 
	expressed as $(k,l)$--bilabelled graph diagrams,
	since $\Aut(G)$ is a subgroup
	of $S_n$.
\end{remark}

We use all of the results given above to obtain a procedure 
for constructing 
the weight matrix for an $\Aut(G)$-equivariant linear layer function
from $(\mathbb{R}^{n})^{\otimes k}$ to $(\mathbb{R}^{n})^{\otimes l}$
from isomorphism classes of $(k,l)$--bilabelled graphs.
In the procedure,
we create all $(q,0)$--bilabelled graph diagrams
that are appropriate for the graph $G$, 
where $q = k+l$, and then use Frobenius duality 
and the functor $\mathcal{F}^G$
to obtain a spanning set of matrices for
$\Hom_{\Aut(G)}((\mathbb{R}^{n})^{\otimes k},
(\mathbb{R}^{n})^{\otimes l})$.
To capture all of the $(q,0)$--bilabelled graph diagrams that are appropriate for $G$,
we first need to calculate the length of the longest path $m$ 
between any two (not necessarily distinct) vertices in $G$ 
where each edge in $G$ is traversed at most once.
This is the longest path that needs to be mapped onto by
a graph homomorphism from the underlying graph 
of a $(q,0)$--bilabelled graph diagram to $G$.
We use $m$ to generate the all of the 
$(q,0)$--bilabelled graph diagrams as follows. 
We focus on the \textbf{non-free red vertices}
since they play a key role in
how the entries of 
a spanning set matrix $X_{\bm{H}}^{G}$ are determined.
We use the following ideas 
to 
generate all of the 
$(q,0)$--bilabelled graph diagrams from the set partitions of $[l+k]$
expressed as $(q,0)$--bilabelled graph diagrams.
\begin{itemize}
	\item Between any pair of non-free red vertices, 
		we must consider all paths of lengths $0$ to $2m$ between
		them consisting solely of free red vertices, since such a path of length
		$2m$ is a double cover of the longest path in $G$.
	\item Starting with a single non-free red vertex that is not connected with any 
		other red vertex, we must consider
		all paths of lengths $0$ to $m$ consisting solely of 
		free red vertices from this non-free red vertex, 
		since such a path of length $m$ is a single cover of the longest path in $G$.
\end{itemize}
		We also need to consider the impact of loops on  
		$(q,0)$--bilabelled graph diagrams if $G$ has loops (see Step 5 of the procedure).
These points
show
that the set of $(q,0)$--bilabelled graph diagrams that we need to 
consider
depends on the graph $G$.




\begin{remark}
	In the Technical Appendix,
	we have provided a number of examples for how to calculate a spanning
	set of $\Hom_{\Aut(G)}((\mathbb{R}^{n})^{\otimes k},
	(\mathbb{R}^{n})^{\otimes l})$ 
	using the procedure,
	for different graphs $G$ and for low
	order tensor powers of $\mathbb{R}^{n}$.
\end{remark}



\section{Limitations and Feasibility}

Given the current limitations of hardware, 
we recognize that
there will be difficulties when implementing the neural networks that are
discussed in this paper.  
Considerable engineering efforts will
be needed to achieve the required scale due to the non-trivial task of
storing high-order tensors and the weight matrices in memory.  
\citet{clebschgordan} demonstrated this
by developing custom CUDA kernels to implement their tensor product-based
neural networks.  However, we anticipate that as computing power continues to
improve, higher-order group equivariant neural networks will become more
prominent in practical applications.  

\section{Conclusion}

We are the first to show how the combinatorics underlying bilabelled graphs
provides the theoretical background for constructing neural networks that are
equivariant to the automorphism group of a graph having $n$ vertices where the
layers are some tensor power of $\mathbb{R}^{n}$.  We found the form of the
learnable, linear, $\Aut(G)$-equivariant layer functions between such tensor
power spaces in the standard basis of $\mathbb{R}^{n}$ by finding a spanning
set for the $\Hom_{\Aut(G)}$--spaces in which these layer functions live.  However, given
that the number of isomorphism classes of $(k,l)$--bilabelled graphs increases
exponentially, both as the number of vertices in the graph $G$ increases and as
$k$ and $l$ increase, resulting in the number of spanning set elements
increasing exponentially, it would be useful to find further ways of reducing
the number of isomorphism classes of $(k,l)$--bilabelled graphs that we need to
consider.  We leave this to future work.

\begin{figure*}[htb!]
	\begin{tcolorbox}[colback=melon!10, colframe=melon!40, coltitle=black, 
		title={\bfseries 
		Procedure: Weight Matrix for
		an $\Aut(G)$-Equivariant Linear Layer Function from 
		$(\mathbb{R}^{n})^{\otimes k}$
		to 
		$(\mathbb{R}^{n})^{\otimes l}$.
		},
		fonttitle=\bfseries
		]
		Calculate all $(q,0)$--bilabelled graph diagrams 
		that are appropriate for $G$ where $q = l+k$, as follows.
	Let $m$ be the maximum length of a path between any two vertices in $G$ 
	such that each edge in $G$ is traversed at most once.
	Then
	\begin{enumerate}
		\item Calculate all set partitions of $\{1, \dots, q=l+k\}$, and express
			them as $(q,0)$--bilabelled graph diagrams. 
			For uniformity, we choose
			to order the blocks in a set partition by their size, 
			in descending order.
			Each block $B_i$ of size $b_i$
			corresponds to a spider diagram 
			having a single red vertex
			and $b_i$ black vertices 
			connected to it with black edges. The vertices
			in the spider diagram for $B_i$
			are labelled with the elements of $B_i$ 
			in ascending order, from left to right.
			If $m = 0$, go to Step $6$. Otherwise:
		\item 
			From all of the 
			diagrams
			found in Step $1$, 
			create all possible $(q,0)$--bilabelled graph diagrams
			that have only internal red edges between red vertices, 
			as follows.
			For each 
			diagram
			$\bm{H}$ found in Step $1$:
			\begin{itemize}
				\item Let $c$ be the number of red vertices in $\bm{H}$. 
					Let $e \coloneqq \frac{c(c-1)}{2}$: this is the number of edges in the complete graph $K_c$ having $c$ vertices.
					Create all possible $e$ length strings having values in $0 \rightarrow 2m$. Each position in the string represents a different pair of distinct red vertices in $\bm{H}$. Remove the all $0$ string.
				\item For each string, create a new $(q,0)$--bilabelled graph diagram as follows:
					for each position in the string, if $t$ is 
					its value,
					then, assuming that the position represents the pair of red vertices $v$ and $w$, insert $t$ new edges between $v$ and $w$ into $\bm{H}$, adding in $t-1$ new red vertices to make these $t$ new edges possible.
		
			\end{itemize}
		\item Also, from all of the 
			diagrams
			found 
			in Step $1$, create all possible $(q,0)$--bilabelled graph 
			diagrams
			that have only external red edges, as follows.
			For each 
			diagram
			$\bm{H}$ found in Step $1$:
			\begin{itemize}
				\item Let $c$ be the number of red vertices in $\bm{H}$. 
					Create all possible $c$ length strings having values in $0 \rightarrow m$, where each position in the string represents a red vertex in $\bm{H}$. Remove the all $0$ string.	
				\item For each string, create a new $(q,0)$--bilabelled graph diagram as follows:
					for each position in the string, if $t$ is 
					its
					value,
					then, assuming that the position represents the red vertex $v$, add $t$ new edges outwards from $v$, adding in $t-1$ new red vertices to make these $t$ new edges possible.
			\end{itemize}
		\item 
			From all of the 
			diagrams
			found 
			in Step $2$, create all possible $(q,0)$--bilabelled graph diagrams having both internal and external red edges, as follows.
			For each 
			diagram
			$\bm{H_I}$ found in Step $2$:
			\begin{itemize}
				\item Let $\bm{H}$ be the $(q,0)$--bilabelled graph diagram from 
					Step $1$ that $\bm{H_I}$ came from. 
					Let $c$ be the number of red vertices in $\bm{H}$,
					and label these vertices in $\bm{H_I}$ as $1, \dots, c$.
					Create an empty set named \textit{originals}, 
					and add only the vertices $1, \dots, c$ in $\bm{H_I}$
					that are not connected to any other red vertex in $\bm{H_I}$.
				\item If \textit{originals} is empty, 
					then no new diagrams come from $\bm{H_I}$. Otherwise, create all possible $c$ length strings,
					indexed by the vertices $1, \dots, c$, 
					allowing 
					only the values in the positions in
					\textit{originals} to range from $0 \rightarrow m$, 
					with the rest being $0$.
					Remove the all $0$ string.
					Now follow the instructions given in Step $3.2$.
			\end{itemize}
		\item Finally, if $G$ has loops, then, for each $(q,0)$--bilabelled graph diagram found
			in Steps $1$--$4$:
			\begin{itemize}
				\item If $c$ is the number of red vertices in $\bm{H}$, label these vertices as $1, \dots, c$, and create all binary strings of length $c$.
					For each string, create a new $(q,0)$--bilabelled graph diagram as follows: for each position $i$ in the string, if the value is $1$, then attach a loop to the red vertex labelled as $i$.
			\end{itemize}
	\end{enumerate}

	The weight matrix is obtained from the set of all 
	$(q,0)$--bilabelled graph diagrams found in Steps $1$--$5$,
	as follows:
	\begin{enumerate}
		\setcounter{enumi}{5}
	\item Apply Frobenius duality to each $(q,0)$--bilabelled graph diagram $\bm{H}$ 
		to obtain
		its form as a $(k,l)$--bilabelled graph diagram. 
			This is the same
		as dragging the labelled black vertices in $\bm{H}$ 
			into two rows of vertices
		such that the vertices in the top row are ordered $1, \dots, l$ 
		and the vertices in the bottom row are ordered $l + 1, \dots, l+k$.
	\item Apply the strict $\mathbb{R}$--linear monoidal functor 
		$\mathcal{F}^G$ to each $(k,l)$--bilabelled graph diagram
		to obtain the spanning set matrices in 
			$\Hom_{\Aut(G)}((\mathbb{R}^{n})^{\otimes k}, 
			(\mathbb{R}^{n})^{\otimes l})$. 
			Remove all duplicate matrices as well as the all zero matrices from this set.
		Weight each matrix that remains in the set by a parameter, 
			and then add them 
		all together to give the overall $\Aut(G)$-equivariant weight matrix.
	\end{enumerate}
	\end{tcolorbox}
  	\label{AutGsummaryprocedure}
\end{figure*}


\section*{Acknowledgements}


This work was funded by the Doctoral
Scholarship for Applied Research which was
awarded to the first author under Imperial College
London's Department of Computing Applied
Research scheme.  This work will form part of
the first author's PhD thesis at Imperial College
London.

\section*{Impact Statement}

This paper presents work that is primarily a theoretical contribution; hence we do not expect profound societal impact in the short term. However, in the medium term, a number of applications may well emerge from the theory having high levels of impact.


\nocite{*} \bibliography{references}
\bibliographystyle{icml2024}

\newpage \appendix \onecolumn

\section{Proofs}

To prove Propositions \ref{tenscatduals}, \ref{catbilgraph}, 
and \ref{catGhommat},
we first need to define
a strict $\mathbb{R}$--linear monoidal category. 
We assume throughout that all categories are \textbf{locally small}, which 
means that the collection of morphisms between any two objects is a set.
In fact, all of the categories that we consider throughout have 
morphism sets that are vector spaces.
Hence, the morphisms between objects become linear maps.

\begin{definition}  \label{categorystrictmonoidal}
	A category $\mathcal{C}$ is said to be \textbf{strict monoidal} if it comes with
		a bifunctor $\otimes: \mathcal{C} \times \mathcal{C} \rightarrow \mathcal{C}$, called the tensor product, and
	a unit object $\mathds{1}$,
	such that, for all objects $X, Y, Z$ in $\mathcal{C}$, we have that
	\begin{equation}
		(X \otimes Y) \otimes Z = X \otimes (Y \otimes Z)
	\end{equation}
	\begin{equation}
		(\mathds{1} \otimes X) = X = (X \otimes \mathds{1})
	\end{equation}
	and, for all morphisms $f, g, h$ in $\mathcal{C}$, we have that
	\begin{equation} \label{assocbifunctor}
		(f \otimes g) \otimes h = f \otimes (g \otimes h)
	\end{equation}
	\begin{equation}
		(1_\mathds{1} \otimes f) = f = (f \otimes 1_\mathds{1})
	\end{equation}
	where $1_\mathds{1}$ is the identity morphism $\mathds{1} \rightarrow \mathds{1}$.
	We often use the tuple
	$(\mathcal{C}, \otimes_\mathcal{C}, \mathds{1}_\mathcal{C})$
	to refer to the strict monoidal category $\mathcal{C}$.
\end{definition}

We can assume that all monoidal categories are \textbf{strict} 
(nonstrict monoidal categories would have isomorphisms 
where there are equalities in Definition \ref{categorystrictmonoidal}) 
owing to a technical result known as Mac Lane's Coherence Theorem. 
See \citet{maclane} for more details.

\begin{definition} \label{categorylinear}
	A category $\mathcal{C}$ is said to be $\mathbb{R}$\textbf{--linear} if,
	for any two objects $X, Y$ in $\mathcal{C}$, the morphism space 
	$\Hom_{\mathcal{C}}(X,Y)$ is a vector space over $\mathbb{R}$, and
	the composition of morphisms is $\mathbb{R}$--bilinear.
\end{definition}

Combining Definitions \ref{categorystrictmonoidal} and \ref{categorylinear},
we get
\begin{definition}
	A category $\mathcal{C}$ is said to be \textbf{strict} $\mathbb{R}$\textbf{--linear monoidal} if it is a category that is both strict monoidal and $\mathbb{R}$--linear, such that the bifunctor $\otimes$ is $\mathbb{R}$--bilinear.
\end{definition}

\begin{proof}[Proof of Proposition \ref{tenscatduals}]
	$\mathcal{C}(G)$ can immediately be seen to be a strict monoidal category
	by the associativity of the tensor product on vector spaces and 
	by the associativity of the tensor product on linear maps between tensor spaces.

	$\mathcal{C}(G)$ is $\mathbb{R}$--linear because the morphism space 
	$\Hom_{G}((\mathbb{R}^{n})^{\otimes k},
	(\mathbb{R}^{n})^{\otimes l})$, for any two objects 
	$(\mathbb{R}^{n})^{\otimes k}$ and
	$(\mathbb{R}^{n})^{\otimes l}$ in $\mathcal{C}(G)$,
	is a vector space over $\mathbb{R}$ 
	and the composition of morphisms is $\mathbb{R}$-bilinear because
	composition is $\mathbb{R}$-bilinear for linear maps on vector spaces.
	It is also clear that the bifunctor $\otimes$ is $\mathbb{R}$--bilinear
	since it is the standard tensor product for vector spaces.
\end{proof}

\begin{proof} [Proof of Proposition \ref{catbilgraph}]
	$\mathcal{G}$ is a strict monoidal category
	because the bifunctor on objects reduces to the addition operation on 
	natural numbers, which is associative, and
	the bifunctor on morphisms is 
	the tensor product of isomorphism classes of 
	bilabelled graphs given in Definition \ref{tensorprodgraph}, which 
	is associative because both the concatenation of tuples and the (set) unions
	of graphs are associative operations.

	$\mathcal{G}$ is $\mathbb{R}$--linear because the morphism space 
	between any two objects is by definition a vector space,
	and the composition of morphisms is $\mathbb{R}$-bilinear 
	by definition.
	For the same reason, the bifunctor is also $\mathbb{R}$--bilinear.
\end{proof}

\begin{proof} [Proof of Proposition \ref{catGhommat}]
	This is effectively the same proof as the one that is given 
	for Proposition \ref{tenscatduals},
	except we replace linear maps by matrices.
\end{proof}

In order to prove Theorem \ref{graphfunctor},
we first need to recall 
the definition of a strict $\mathbb{R}$--linear monoidal functor. 
	
\begin{definition} \label{monoidalfunctordefn}
	Suppose that
	$(\mathcal{C}, \otimes_\mathcal{C}, \mathds{1}_\mathcal{C})$
	and
	$(\mathcal{D}, \otimes_\mathcal{D}, \mathds{1}_\mathcal{D})$
	are two strict $\mathbb{R}$--linear monoidal categories.

	A \textbf{strict} $\mathbb{R}$\textbf{--linear monoidal functor} 
	from $\mathcal{C}$ to $\mathcal{D}$ is a functor 
	$\mathcal{F}: \mathcal{C} \rightarrow \mathcal{D}$ 
	such that
	\begin{enumerate}
		\item for all objects $X, Y$ in $\mathcal{C}$,
			$\mathcal{F}(X \otimes_\mathcal{C} Y) =
			\mathcal{F}(X) \otimes_\mathcal{D} \mathcal{F}(Y)$
		\item for all morphisms $f, g$ in $\mathcal{C}$,
			$\mathcal{F}(f \otimes_\mathcal{C} g) =
			\mathcal{F}(f) \otimes_\mathcal{D} \mathcal{F}(g)$
		\item $\mathcal{F}(\mathds{1}_\mathcal{C}) = \mathds{1}_\mathcal{D}$, and
		\item for all objects $X, Y$ in $\mathcal{C}$, the map
		\begin{equation} \label{maphomsets}
			\Hom_{\mathcal{C}}(X,Y) 
			\rightarrow 
			\Hom_{\mathcal{D}}(\mathcal{F}(X),\mathcal{F}(Y))
		\end{equation}
		given by
		$f \mapsto \mathcal{F}(f)$
		is $\mathbb{R}$--linear.
	\end{enumerate}
\end{definition}

\begin{proof}[Proof of Theorem \ref{graphfunctor}]
	Let $G$ be a graph having $n$ vertices. We show each of the 
	four conditions of Definition \ref{monoidalfunctordefn} in turn.
\begin{enumerate}
	\item Let $k, l$ be any two objects in $\mathcal{G}$. Then 
	\begin{equation}
		\mathcal{F}^G(k \otimes l) 
		 = \mathcal{F}^G(k + l) 
		 = (\mathbb{R}^{n})^{\otimes k+l} 
		 = (\mathbb{R}^{n})^{\otimes k} \otimes (\mathbb{R}^{n})^{\otimes l} 
		 = \mathcal{F}^G(k) \otimes \mathcal{F}^G(l)
	\end{equation}
	\item This is Lemma \ref{gHomtensorprod}.
	\item It is clear from the statement of the theorem that $\mathcal{F}^G$ 
		sends the unit object $0$ in $\mathcal{G}$ to $\mathbb{R}$, 
		which is the unit object in $\mathcal{C}^G$.
	\item This is immediate from the definition of a $G$-homomorphism matrix.
\end{enumerate}
	It is clear that the functor $\mathcal{F}^G$ is full, once again by 
	the definition of a $G$-homomorphism matrix.
\end{proof}

\begin{proof}[Proof of Proposition \ref{disconngraph}]
	Suppose that 
	$\bm{H} = (H, \bm{k}, \bm{l})$.
	Let $H_2$ be the subgraph consisting of the subset of free vertices
	(including the edges that are solely between these vertices) 
	that are entirely disconnected from any vertices that are 
	in a connected component containing a non-free vertex.
	Define $H_1$ to be the graph that is $H$ without $H_2$.

	Then, by construction, $H$ is a disjoint union 
	of $H_1$ and $H_2$, as graphs,
	and, if we define
	$\bm{H_1} \coloneqq (H_1, \bm{k}, \bm{l})$, 
	then it is clear that
	\begin{equation}
		X_{\bm{H}}^G
		= 
		c_{\bm{H_2}}^G
		X_{\bm{H_1}}^G
	\end{equation}
	where $c_{\bm{H_2}}^G$ is a constant denoting the number of graph 
	homomorphisms from $H_2$ to $G$, as required.
\end{proof}

\begin{proof}[Proof of Proposition \ref{FrobDuality}]
	Firstly, there is 
	an $\mathbb{R}$--linear isomorphism
	\begin{equation}
		\Hom_{\mathcal{G}}(k,l)
		\rightarrow
		\Hom_{\mathcal{G}}(l+k,0)
	\end{equation}
	that is given on bilabelled graph diagrams by
	\begin{equation}\label{frobdual}
		\scalebox{0.6}{\input{graphfrob1.tikz}}
	\end{equation}
	with inverse given by
	
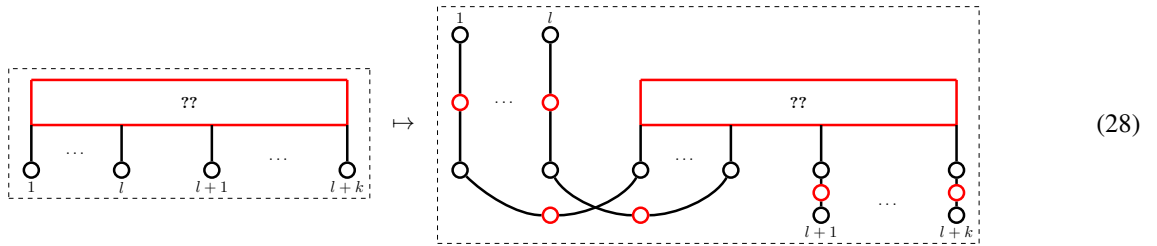
\begin{equation}\label{frobdualinv}
		\begin{aligned}
			\scalebox{0.6}{\input{graphfrob2.tikz}}
		\end{aligned}
	\end{equation}

	Similarly, for any graph $G$ having $n$ vertices,
	there is 
	an $\mathbb{R}$-linear isomorphism
	\begin{equation}
		\Hom_{\Aut(G)}((\mathbb{R}^{n})^{\otimes
		k}, (\mathbb{R}^{n})^{\otimes l})
		\rightarrow
		\Hom_{\Aut(G)}((\mathbb{R}^{n})^{\otimes
		l+k}, \mathbb{R})
	\end{equation}
	that is given by
	\begin{equation}
		X_{\bm{H}}^G 
		\mapsto
		\mathcal{F}^G([\bm{MM}_l])
		\circ
		(Id^{\otimes k} \otimes X_{\bm{H}}^G)
	\end{equation}
	where $Id$ is the $n \times n$ identity matrix,
	with inverse given by
	\begin{equation}
		X_{\bm{H}}^G 
		\mapsto
		(Id^{\otimes l} \otimes X_{\bm{H}}^G)
		\circ
		(\mathcal{F}^G([\bm{UU}_l]) \otimes Id^{\otimes k})
	\end{equation}
	Here, $\bm{MM}_l$ is the $(2l,0)$--bilabelled graph diagram
	\begin{equation}
		\scalebox{0.6}{\input{graphfrobhat.tikz}}
	\end{equation}
	and $\bm{UU}_l$ is its involution, as defined in Definition
	\ref{involutiongraph}.

	Since, for any graph $G$ having $n$ vertices,
	the functor $\mathcal{F}^G$ is strict monoidal, 
	by Theorem \ref{graphfunctor},
	we get that the following diagram commutes:

	\begin{equation}
		\begin{aligned}
			\input{commsquare.tikz}
		\end{aligned}
	\end{equation}

	Since $\mathcal{F}^G$ gives a bijective correspondence
	between all isomorphism classes of bilabelled graphs
	and the spanning set elements for $\Hom_{\Aut(G)}$, as shown
	in Theorem \ref{graphmainthrm},
	we can use the commuting square to obtain a spanning set for
	$\Hom_{\Aut(G)}((\mathbb{R}^{n})^{\otimes
	k+l}, \mathbb{R})$
	from the spanning set for
	$\Hom_{\Aut(G)}((\mathbb{R}^{n})^{\otimes
	k}, (\mathbb{R}^{n})^{\otimes l})$.
	
	Now, fixing $q, m \geq 0$ such that
	$q + m = k + l$,
	we can run the arrows of the commuting square
	in reverse
	to obtain a spanning set for 
	$\Hom_{\Aut(G)}((\mathbb{R}^{n})^{\otimes
	q}, (\mathbb{R}^{n})^{\otimes m})$,
	as required.
\end{proof}

\section{Recovery of the Characterisation of the Equivariant, Linear
Maps for the Symmetric Group} \label{recoverysymm}

We saw in Example \ref{autocompletegraph} that the automorphism group of the
complement of the complete graph, $\Aut(\overline{K_n})$, is the symmetric
group $S_n$.  As a result of applying Theorem \ref{graphmainthrm} to this case,
we can recover the diagram basis for $\Hom_{S_n}((\mathbb{R}^{n})^{\otimes k},
(\mathbb{R}^{n})^{\otimes l})$ for any non-negative integers $k$ and $l$ that first
appeared in \citet{godfrey}.

To recover the diagram basis, we reduce the set of isomorphism classes of
$(k,l)$--bilabelled graphs $\bm{H}$ that we need to consider in
(\ref{autspanres}), as follows.  Firstly, we only need to consider isomorphism
classes of $(k,l)$--bilabelled graphs $\bm{H}$ whose underlying graphs $H$ are
edgeless, since if $H$ has either some edge between two distinct vertices or a
loop, then for there to be a graph homomorphism from $H$ to $\overline{K_n}$,
the graph $\overline{K_n}$ would need to have an edge between two distinct
vertices or a loop. 
Secondly, we only need to consider isomorphism classes of $(k,l)$--bilabelled
graphs $\bm{H}$ whose underlying graph $H$ is edgeless, having at most $n$
vertices, since the image under $\mathcal{F}^{G}$ of any $(k,l)$--bilabelled
graph whose underlying graph is edgeless having more than $n$ vertices would
correspond to a scalar multiple of the image of a $(k,l)$--bilabelled graph
whose underlying graph is edgeless having at most $n$ vertices, Thirdly, we
only need to consider isomorphism classes of $(k,l)$--bilabelled graphs
$\bm{H}$ whose underlying graph $H$ is edgeless having at most $n$ vertices
where no vertex is left \textit{free} in $\bm{H}$, that is, there is no red
vertex in $H$ that is not attached by a black edge to some black vertex, by
Proposition \ref{disconngraph}.

But this subset of isomorphism classes of $(k,l)$--bilabelled graphs is
precisely all set partitions of $\{1, \dots, l+k\}$ having at most $n$
blocks!  By construction, the image of this subset under $\mathcal{F}^{G}$ is a
spanning set of $\Hom_{S_n}((\mathbb{R}^{n})^{\otimes k},
(\mathbb{R}^{n})^{\otimes l})$ and, by a dimension count, the image, in fact,
forms a basis, which is precisely the diagram basis of \citet{godfrey}.

\section{Spanning Set Examples} 

\begin{example}[Diagram Basis for
$
\Hom_{S_n}(\mathbb{R}^{n},
\mathbb{R}^{n})
$]\label{diagbasis11}

Suppose we take the complement of the complete graph having $n$ vertices,
$\overline{K_n}$. There is only one possible way to label the vertices
of this graph, up to all automorphisms: 
\begin{center} 
	\scalebox{0.9}{\input{compKn.tikz}} 
\end{center} 
In Example \ref{autocompletegraph},
we saw that its automorphism group, $\Aut(\overline{K_n})$, is the symmetric group $S_n$.

Since the order of each tensor power space is $1$, by the arguments in Section
	\ref{recoverysymm}, it is enough to consider the isomorphism classes of
	$(1,1)$--bilabelled graphs $\bm{H}$ whose underlying graph $H$ is
	edgeless having at most $n$ vertices where no vertex is left free in
	$\bm{H}$.  Because only at most two vertices cannot be left free for
	$(1,1)$--bilabelled graphs $\bm{H}$, we only need to consider
	isomorphism classes of $(1,1)$--bilabelled graphs $\bm{H}$ whose
	underlying graph $H$ is edgeless having at most two vertices.

This leads to two possibilities, namely $[\bm{H_1}]$ and $[\bm{H_2}]$, where
	$\bm{H_1} = (K_1, (1), (1))$ and where $\bm{H_2} = (\overline{K_2},
	(1), (2))$.  
	They are given in the left hand
	column of Figure \ref{S4diagbasis}.  (Note, for example, that the
	bilabelled graph $(\overline{K_2}, (2), (1))$ is in the same
	isomorphism class as $\bm{H_2}$.)

In the right hand column of Figure \ref{S4diagbasis}, we show the corresponding
	$G$-homomorphism matrices for the case where $G = \overline{K_4}$, that
	is, $n = 4$, with the case for general $n$ being very similar.  Here we
	have used Definition \ref{GHomMatrix} to find the entries of each of
	the matrices. 
To be clear, for general $n$, the $(i,j)$-entry of
	$X_{\bm{H_1}}^{\overline{K_n}}$,
	by definition, is equal to the
	number of graph homomorphisms from $K_1$ to $\overline{K_n}$ such that
	$1$ is mapped to $i$ and $1$ is mapped to $j$, which is $\delta_{i,j}$.
	By contrast, the $(i,j)$-entry of $X_{\bm{H_2}}^{\overline{K_n}}$,
	is equal to the number of graph homomorphisms from $\overline{K_2}$ to
	$\overline{K_n}$ such that $2$ is mapped to $i$ and $1$ is mapped to
	$j$, which is $1$ for all $i, j \in [n]$.

In the middle column, we have also shown the equivalent set partition diagrams
	that appear in \citet{pearcecrump} to highlight how the $(1,1)$--bilabelled
	graph diagrams are related to set partition diagrams in this case.
	This is also consistent with the procedure that is given in the
	orange box since the longest path in $\overline{K_n}$ is $0$. Hence,
	only the set partitions of $\{1, 2\}$, expressed as $(2,0)$--bilabelled
	graph diagrams that are then mapped under Frobenius duality to
	$(1,1)$--bilabelled graph diagrams, are needed to obtain a spanning set ---
	which is actually a basis ---
	of $\Hom_{S_n}(\mathbb{R}^{n}, \mathbb{R}^{n})$.

\begin{figure}[tb]
\begin{center}
\begin{tblr}{
		colspec = {X[c]X[c]X[c]},
  stretch = 0,
  rowsep = 5pt,
  hlines = {1pt},
  vlines = {1pt},
}
	{$(1,1)$--Bilabelled Graph Diagram $\bm{H}$} & 
	{Set Partition Diagram} & 
	{Standard Basis Element $X_{\bm{H}}^{\overline{K_4}}$ } \\
	\scalebox{0.6}{\input{graph11sq1.tikz}} & 
	\scalebox{0.6}{\input{orbit11sq1.tikz}} & 
	\scalebox{0.85}{
	$
	\NiceMatrixOptions{code-for-first-row = \scriptstyle \color{blue},
                   	   code-for-first-col = \scriptstyle \color{blue}
	}
	\begin{bNiceArray}{*{2}{c}*{2}{c}}[first-row,first-col]
				& 1 	& 2	& 3	& 4 \\
		1		& 1	& 0	& 0	& 0	\\
		2		& 0	& 1	& 0	& 0	\\
		3		& 0	& 0	& 1	& 0	\\
		4		& 0	& 0	& 0	& 1	
	\end{bNiceArray}
	$}
	\\
	\scalebox{0.6}{\input{graph11sq2.tikz}}	& 
	\scalebox{0.6}{\input{orbit11sq2.tikz}} & 
	\scalebox{0.85}{
	$
	\NiceMatrixOptions{code-for-first-row = \scriptstyle \color{blue},
                   	   code-for-first-col = \scriptstyle \color{blue}
	}
	\begin{bNiceArray}{*{2}{c}*{2}{c}}[first-row,first-col]
				& 1 	& 2	& 3	& 4 	\\
		1		& 1	& 1	& 1	& 1	\\
		2		& 1	& 1	& 1	& 1	\\
		3		& 1	& 1	& 1	& 1	\\
		4		& 1	& 1	& 1	& 1	
	\end{bNiceArray}
	$}
\end{tblr}
	\caption{
		We use Theorem \ref{graphmainthrm}
		to obtain the diagram basis of
	$\Hom_{S_4}(\mathbb{R}^{4}, \mathbb{R}^{4})$
	from all of the 
	$(1,1)$--bilabelled graph diagrams.}
	\label{S4diagbasis}
	\end{center}
\end{figure}
\end{example}

\begin{example}[Basis for
$
	\Hom_{D_4}(\mathbb{R}^{4},
	\mathbb{R}^{4})
$] 
\label{D4k1l1ex}

	In Example \ref{autotwocompletetwo}, we said that the automorphism group of two
copies of the complete graph on two vertices, $\Aut(2K_2)$, is
isomorphic to the dihedral group $D_4$ of order $8$.

There are three different ways to label the vertices of the graph $2K_2$, up to all automorphisms, namely:
	\begin{center}
		\scalebox{0.9}{\input{comp2K2.tikz}}
	\end{center}
We will refer to these graphs as $A$, $B$ and $C$, respectively, throughout the rest of this example.

The automorphism group of each graph is isomorphic to $D_4$, but each automorphism group is a
different embedding of $D_4$ in $S_4$.
Said differently, while each automorphism group is isomorphic to $D_4$, the elements in each automorphism group are not the same when thought of as elements of $S_4$.
Specifically, the group corresponding to the first diagram is 
$\langle (1423), (13)(24) \rangle$, 
whereas the second is
$\langle (1432), (12)(34) \rangle$,
and the third is 
$\langle (1342), (12)(34) \rangle$. 

Consequently, when we would like to find
a spanning set for
$
	\Hom_{D_4}(\mathbb{R}^{4},
	\mathbb{R}^{4})
$,
we need to consider which embedding of $D_4$ in $S_4$ we are referring to, or, more simply, which labels we have used for the graph $2K_2$.
As a result, we will obtain a spanning set (which we will show is actually a basis) for
$
	\Hom_{D_4}(\mathbb{R}^{4},
	\mathbb{R}^{4})
$,
one for each embedding of $D_4$ in $S_4$. 
Note that these spanning sets, whilst different, will all be isomorphic under a change of basis that reorders the standard basis vectors in
$\mathbb{R}^{4}$. 
This is a consequence of all of the automorphism groups being isomorphic to each other.

We present the bases for each instance of $2K_2$ in Figure \ref{D4diagbasis}.
			
\begin{figure}[tb]
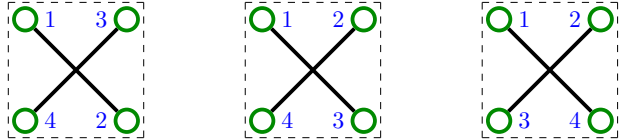

\begin{center}
\begin{tblr}{
		colspec = {X[c]X[c]X[c]X[c]},
  stretch = 0,
  rowsep = 5pt,
  hlines = {1pt},
  vlines = {1pt},
}
	{$(1,1)$--Bilabelled Graph Diagram $\bm{H}$} & 
	{Standard Basis Element $X_{\bm{H}}^{A}$ } &
	{Standard Basis Element $X_{\bm{H}}^{B}$ } &
	{Standard Basis Element $X_{\bm{H}}^{C}$ } \\
	\scalebox{0.6}{\input{graph11sq1.tikz}} & 
	\scalebox{0.85}{
	$
	\NiceMatrixOptions{code-for-first-row = \scriptstyle \color{blue},
                   	   code-for-first-col = \scriptstyle \color{blue}
	}
	\begin{bNiceArray}{*{2}{c}*{2}{c}}[first-row,first-col]
				& 1 	& 2	& 3	& 4 \\
		1		& 1	& 0	& 0	& 0	\\
		2		& 0	& 1	& 0	& 0	\\
		3		& 0	& 0	& 1	& 0	\\
		4		& 0	& 0	& 0	& 1	
	\end{bNiceArray}
	$} &
	\scalebox{0.85}{
	$
	\NiceMatrixOptions{code-for-first-row = \scriptstyle \color{blue},
                   	   code-for-first-col = \scriptstyle \color{blue}
	}
	\begin{bNiceArray}{*{2}{c}*{2}{c}}[first-row,first-col]
				& 1 	& 2	& 3	& 4 \\
		1		& 1	& 0	& 0	& 0	\\
		2		& 0	& 1	& 0	& 0	\\
		3		& 0	& 0	& 1	& 0	\\
		4		& 0	& 0	& 0	& 1	
	\end{bNiceArray}
	$} &
	\scalebox{0.85}{
	$
	\NiceMatrixOptions{code-for-first-row = \scriptstyle \color{blue},
                   	   code-for-first-col = \scriptstyle \color{blue}
	}
	\begin{bNiceArray}{*{2}{c}*{2}{c}}[first-row,first-col]
				& 1 	& 2	& 3	& 4 \\
		1		& 1	& 0	& 0	& 0	\\
		2		& 0	& 1	& 0	& 0	\\
		3		& 0	& 0	& 1	& 0	\\
		4		& 0	& 0	& 0	& 1	
	\end{bNiceArray}
	$}
	\\
	\scalebox{0.6}{\input{graph11sq2.tikz}}	& 
	\scalebox{0.85}{
	$
	\NiceMatrixOptions{code-for-first-row = \scriptstyle \color{blue},
                   	   code-for-first-col = \scriptstyle \color{blue}
	}
	\begin{bNiceArray}{*{2}{c}*{2}{c}}[first-row,first-col]
				& 1 	& 2	& 3	& 4 	\\
		1		& 1	& 1	& 1	& 1	\\
		2		& 1	& 1	& 1	& 1	\\
		3		& 1	& 1	& 1	& 1	\\
		4		& 1	& 1	& 1	& 1	
	\end{bNiceArray}
	$} &
	\scalebox{0.85}{
	$
	\NiceMatrixOptions{code-for-first-row = \scriptstyle \color{blue},
                   	   code-for-first-col = \scriptstyle \color{blue}
	}
	\begin{bNiceArray}{*{2}{c}*{2}{c}}[first-row,first-col]
				& 1 	& 2	& 3	& 4 	\\
		1		& 1	& 1	& 1	& 1	\\
		2		& 1	& 1	& 1	& 1	\\
		3		& 1	& 1	& 1	& 1	\\
		4		& 1	& 1	& 1	& 1	
	\end{bNiceArray}
	$} &
	\scalebox{0.85}{
	$
	\NiceMatrixOptions{code-for-first-row = \scriptstyle \color{blue},
                   	   code-for-first-col = \scriptstyle \color{blue}
	}
	\begin{bNiceArray}{*{2}{c}*{2}{c}}[first-row,first-col]
				& 1 	& 2	& 3	& 4 	\\
		1		& 1	& 1	& 1	& 1	\\
		2		& 1	& 1	& 1	& 1	\\
		3		& 1	& 1	& 1	& 1	\\
		4		& 1	& 1	& 1	& 1	
	\end{bNiceArray}
	$}
	\\
	\scalebox{0.6}{\input{graph11sq3.tikz}}	& 
	\scalebox{0.85}{
	$
	\NiceMatrixOptions{code-for-first-row = \scriptstyle \color{blue},
                   	   code-for-first-col = \scriptstyle \color{blue}
	}
	\begin{bNiceArray}{*{2}{c}*{2}{c}}[first-row,first-col]
				& 1 	& 2	& 3	& 4 	\\
		1		& 0	& 1	& 0	& 0	\\
		2		& 1	& 0	& 0	& 0	\\
		3		& 0	& 0	& 0	& 1	\\
		4		& 0	& 0	& 1	& 0	
	\end{bNiceArray}
	$} &
	\scalebox{0.85}{
	$
	\NiceMatrixOptions{code-for-first-row = \scriptstyle \color{blue},
                   	   code-for-first-col = \scriptstyle \color{blue}
	}
	\begin{bNiceArray}{*{2}{c}*{2}{c}}[first-row,first-col]
				& 1 	& 2	& 3	& 4 	\\
		1		& 0	& 0	& 1	& 0	\\
		2		& 0	& 0	& 0	& 1	\\
		3		& 1	& 0	& 0	& 0	\\
		4		& 0	& 1	& 0	& 0	
	\end{bNiceArray}
	$} &
	\scalebox{0.85}{
	$
	\NiceMatrixOptions{code-for-first-row = \scriptstyle \color{blue},
                   	   code-for-first-col = \scriptstyle \color{blue}
	}
	\begin{bNiceArray}{*{2}{c}*{2}{c}}[first-row,first-col]
				& 1 	& 2	& 3	& 4 	\\
		1		& 0	& 0	& 0	& 1	\\
		2		& 0	& 0	& 1	& 0	\\
		3		& 0	& 1	& 0	& 0	\\
		4		& 1	& 0	& 0	& 0	
	\end{bNiceArray}
	$}
\end{tblr}
	\caption{
	Depending on how $D_4$ is embedded as a subgroup of $S_4$, we obtain a basis
	of $\Hom_{D_4}(\mathbb{R}^{4}, \mathbb{R}^{4})$, where here $D_4$ refers to the specific
	embedding in $S_4$ that is obtained from the different labellings of the vertices of $2K_2$, considered up to
	all automorphisms.}
	\label{D4diagbasis}
	\end{center}
\end{figure}

We focus on finding the basis for the automorphism group of the graph $A$, 
noting that the approach for the other graphs is similar.
We follow the procedure that is given in the orange box.
We see that the longest path $m$ between any two vertices in $A$ is $1$.

\textbf{Step 1:} We calculate all set partitions of $[l+k] = \{1,2\}$ and
express them as $(2,0)$--bilabelled graph diagrams, labelling only the black vertices.

They are given by
\begin{equation} \label{D4setparts}
	\begin{aligned}
		\bm{A_0} = \scalebox{0.6}{\input{graph11sq1line.tikz}}
		\; \text{and} \;
		\bm{B_0} = \scalebox{0.6}{\input{graph11sq2line.tikz}}
	\end{aligned}
\end{equation}

\textbf{Step 2:} From $\bm{A_0}$ and $\bm{B_0}$, we calculate all possible 
$(2,0)$--bilabelled graph diagrams that have only internal red edges between red vertices.

Since the number of red vertices $c$ in $\bm{A_0}$ is $1$,
this implies that the number of edges in the complete graph $K_1$ is $0$,
and so we do not obtain any new $(2,0)$--bilabelled graph diagrams from $\bm{A_0}$.

However, for $\bm{B_0}$:
\begin{itemize}
	\item The number of red vertices, $c$, in $\bm{B_0}$ is $2$. Hence the number of edges in the complete graph $K_2$, $e$, is $1$. 
	\item We now create all possible length $e=1$ strings having values in $0 \rightarrow 2m = 2$. Remove the all $0$ string. 
\end{itemize}
Hence we obtain the length one strings $1$ and $2$.
We use each of these strings to create a new $(2,0)$--bilabelled graph diagram from $\bm{B_0}$,
which we shall call $\bm{B_{1,1}}$ and $\bm{B_{1,2}}$, by inserting $t$ new edges 
between the two red vertices, where $t$ equals $1$ and $2$ respectively.
Hence, $\bm{B_{1,1}}$ and $\bm{B_{1,2}}$ are given as follows:

\begin{equation} \label{D4internaledges}
	\begin{aligned}
		\bm{B_{1,1}} = \scalebox{0.6}{\input{graph11sq3line.tikz}}
		\; \text{;} \;
		\bm{B_{1,2}} = \scalebox{0.6}{\input{graph11sq4line.tikz}}
	\end{aligned}
\end{equation}

\textbf{Step 3:} From $\bm{A_0}$ and $\bm{B_0}$,
we now calculate all possible $(2,0)$--bilabelled graph diagrams that
have only external red edges between red vertices.

For $\bm{A_0}$, as the number of red vertices is $1$, we create all length one strings
having entries in $0$ to $m=1$, removing the all $0$ string. We create a new
$(2,0)$--bilabelled graph diagram from the string $1$ by adding $1$ new red edge outwards from
the red vertex in $\bm{A_0}$, adding in a new red vertex to make this new red 
edge possible.
Hence we obtain
\begin{equation} \label{D4step1externaledges1}
	\begin{aligned}
		\bm{A_{2}} = \scalebox{0.6}{\input{graph11sq5line.tikz}}
	\end{aligned}
\end{equation}

For $\bm{B_0}$, as the number of red vertices is $2$, we create all length two strings
having entries in $0$ to $m=1$, removing the all $0$ string. Hence we create three new
$(2,0)$--bilabelled graph diagrams from the strings $01$, $10$ and $11$, by adding $1$ 
new red edge outwards from the red vertices in $\bm{A_0}$ corresponding to the ones in the string, adding in a new red vertex for each new red edge.
Hence we obtain
\begin{equation} \label{D4step1externaledges2}
	\begin{aligned}
		\bm{B_{2,1}} = \scalebox{0.6}{\input{graph11sq6line.tikz}}
		\; \text{;} \;
		\bm{B_{2,2}} = \scalebox{0.6}{\input{graph11sq7line.tikz}}
		\; \text{;} \;
		\bm{B_{2,3}} = \scalebox{0.6}{\input{graph11sq8line.tikz}}
	\end{aligned}
\end{equation}

\textbf{Step 4:} We calculate all possible $(2,0)$--bilabelled graph diagrams that
have external red edges from $\bm{B_{1,1}}$ and $\bm{B_{2,2}}$.

We do not create any new $(2,0)$--bilabelled graph diagrams in this step as 
the set \textit{originals} is empty for 
$\bm{B_{1,1}}$ and $\bm{B_{2,2}}$.

\textbf{Step 5:} As $A$ does not have any loops, we immediately move onto Step 6.

\textbf{Step 6:} We now apply Frobenius duality to each $(2,0)$--bilabelled graph diagram
found in Steps $1$ to $5$ inclusive to obtain their form as $(1,1)$--bilabelled graph diagrams.

This is equivalent to dragging the black vertex labelled $1$ up to the top row.
At this stage, we choose to arbitrarily label the red vertices as well.
Using the same names for the bilabelled graph diagrams, we obtain:

\begin{equation} \label{D411i}
	\begin{aligned}
		\bm{A_0} = \scalebox{0.6}{\input{graph11sq1.tikz}}
		\; \text{;} \;
		\bm{B_0} = \scalebox{0.6}{\input{graph11sq2.tikz}}
		\; \text{;} \;
		\bm{B_{1,1}} = \scalebox{0.6}{\input{graph11sq3.tikz}}
		\; \text{;} \;
		\bm{B_{1,2}} = \scalebox{0.6}{\input{graph11sq4.tikz}}
	\end{aligned}
\end{equation}

\begin{equation} \label{D411ii}
	\begin{aligned}
		\bm{A_{2}} = \scalebox{0.6}{\input{graph11sq5.tikz}}	
		\; \text{;} \;
		\bm{B_{2,1}} = \scalebox{0.6}{\input{graph11sq6.tikz}}
		\; \text{;} \;
		\bm{B_{2,2}} = \scalebox{0.6}{\input{graph11sq7.tikz}}
		\; \text{;} \;
		\bm{B_{2,3}} = \scalebox{0.6}{\input{graph11sq8.tikz}}
	\end{aligned}
\end{equation}

\textbf{Step 7:} We calculate the $A$-homomorphism matrices 
that
correspond to the $(1,1)$--bilabelled graph diagrams given in Step 6.
We remove all duplicate matrices as well as any all zero matrices, 
weight those that remain, and then add
them together to obtain the weight matrix for an $\Aut(A) \cong D_4$-equivariant linear
layer function from 
$\mathbb{R}^4$ to 
$\mathbb{R}^4$.

The $A$-homomorphism matrices that correspond to the $(1,1)$--bilabelled graph diagrams
given in (\ref{D411i}) are

	\begin{equation}
	\NiceMatrixOptions{code-for-first-row = \scriptstyle \color{blue},
                   	   code-for-first-col = \scriptstyle \color{blue}
	}
	\begin{bNiceArray}{*{4}{c}}[first-row,first-col]
				& 1 	& 2	& 3	& 4	 	\\
		1		& 1	& 0	& 0	& 0		\\
		2		& 0	& 1	& 0	& 0		\\
		3		& 0	& 0	& 1	& 0		\\	
		4		& 0	& 0	& 0	& 1
	\end{bNiceArray}
				\; \text{; } \;
	\NiceMatrixOptions{code-for-first-row = \scriptstyle \color{blue},
                   	   code-for-first-col = \scriptstyle \color{blue}
	}
	\begin{bNiceArray}{*{4}{c}}[first-row,first-col]
				& 1 	& 2	& 3	& 4	 	\\
		1		& 1	& 1	& 1	& 1		\\
		2		& 1	& 1	& 1	& 1		\\
		3		& 1	& 1	& 1	& 1		\\	
		4		& 1	& 1	& 1	& 1
	\end{bNiceArray}
				\; \text{; } \;
	\NiceMatrixOptions{code-for-first-row = \scriptstyle \color{blue},
                   	   code-for-first-col = \scriptstyle \color{blue}
	}
	\begin{bNiceArray}{*{4}{c}}[first-row,first-col]
				& 1 	& 2	& 3	& 4	 	\\
		1		& 0	& 1	& 0	& 0		\\
		2		& 1	& 0	& 0	& 0		\\
		3		& 0	& 0	& 0	& 1		\\	
		4		& 0	& 0	& 1	& 0
	\end{bNiceArray}
				\; \text{and } \;
	\NiceMatrixOptions{code-for-first-row = \scriptstyle \color{blue},
                   	   code-for-first-col = \scriptstyle \color{blue}
	}
	\begin{bNiceArray}{*{4}{c}}[first-row,first-col]
				& 1 	& 2	& 3	& 4	 	\\
		1		& 1	& 0	& 0	& 0		\\
		2		& 0	& 1	& 0	& 0		\\
		3		& 0	& 0	& 1	& 0		\\	
		4		& 0	& 0	& 0	& 1
	\end{bNiceArray}
	\end{equation}

whereas the $A$-homomorphism matrices that correspond to the $(1,1)$--bilabelled graph diagrams given in (\ref{D411ii}) are
	\begin{equation} 
	\NiceMatrixOptions{code-for-first-row = \scriptstyle \color{blue},
                   	   code-for-first-col = \scriptstyle \color{blue}
	}
	\begin{bNiceArray}{*{4}{c}}[first-row,first-col]
				& 1 	& 2	& 3	& 4	 	\\
		1		& 1	& 0	& 0	& 0		\\
		2		& 0	& 1	& 0	& 0		\\
		3		& 0	& 0	& 1	& 0		\\	
		4		& 0	& 0	& 0	& 1
	\end{bNiceArray}
				\; \text{; } \;
	\NiceMatrixOptions{code-for-first-row = \scriptstyle \color{blue},
                   	   code-for-first-col = \scriptstyle \color{blue}
	}
	\begin{bNiceArray}{*{4}{c}}[first-row,first-col]
				& 1 	& 2	& 3	& 4	 	\\
		1		& 1	& 1	& 1	& 1		\\
		2		& 1	& 1	& 1	& 1		\\
		3		& 1	& 1	& 1	& 1		\\	
		4		& 1	& 1	& 1	& 1
	\end{bNiceArray}
				\; \text{; } \;
	\NiceMatrixOptions{code-for-first-row = \scriptstyle \color{blue},
                   	   code-for-first-col = \scriptstyle \color{blue}
	}
	\begin{bNiceArray}{*{4}{c}}[first-row,first-col]
				& 1 	& 2	& 3	& 4	 	\\
		1		& 1	& 1	& 1	& 1		\\
		2		& 1	& 1	& 1	& 1		\\
		3		& 1	& 1	& 1	& 1		\\	
		4		& 1	& 1	& 1	& 1
	\end{bNiceArray}
				\; \text{and } \;
	\NiceMatrixOptions{code-for-first-row = \scriptstyle \color{blue},
                   	   code-for-first-col = \scriptstyle \color{blue}
	}
	\begin{bNiceArray}{*{4}{c}}[first-row,first-col]
				& 1 	& 2	& 3	& 4	 	\\
		1		& 1	& 1	& 1	& 1		\\
		2		& 1	& 1	& 1	& 1		\\
		3		& 1	& 1	& 1	& 1		\\	
		4		& 1	& 1	& 1	& 1
	\end{bNiceArray}
	\end{equation}
respectively.

Clearly, only the $A$-homomorphism matrices corresponding to
$\bm{A_0}$, $\bm{B_0}$ and $\bm{B_{1,1}}$ are unique, and so
they form a spanning set for
$
	\Hom_{D_4}(\mathbb{R}^{4},
	\mathbb{R}^{4})
$,
for $D_4 = \Aut(A) = \langle (1423), (12)(34) \rangle$, considered as a subgroup of $S_4$.

Hence the weight matrix that we obtain for an $\Aut(A) \cong D_4$-equivariant linear
layer function from 
$\mathbb{R}^4$ to 
$\mathbb{R}^4$ is
\begin{equation}
	\NiceMatrixOptions{code-for-first-row = \scriptstyle \color{blue},
                   	   code-for-first-col = \scriptstyle \color{blue}
	}
	\begin{bNiceArray}{*{4}{c}}[first-row,first-col]
				& 1 	& 2	& 3	& 4	 	\\
		1		& \lambda_{1,2}	& \lambda_{2,3}	& \lambda_2	& \lambda_2		\\
		2		& \lambda_{2,3}	& \lambda_{1,2}	& \lambda_2	& \lambda_2		\\
		3		& \lambda_2	& \lambda_2	& \lambda_{1,2}	& \lambda_{2,3}		\\	
		4		& \lambda_2	& \lambda_2	& \lambda_{2,3}	& \lambda_{1,2}
	\end{bNiceArray}
\end{equation}
for weights $\lambda_1, \lambda_2, \lambda_3 \in \mathbb{R}$, where
$\lambda_{i,j}$ means $\lambda_i + \lambda_j$.

Note that, for this example, each spanning set corresponding to an instance of $2K_2$ happens to be a basis of
$\Hom_{D_4}(\mathbb{R}^{4}, \mathbb{R}^{4})$.
Indeed, for each instance of $2K_2$, the subset consisting of the first two matrices in each column in Figure \ref{D4diagbasis}
is a basis for 
	$\Hom_{S_4}(\mathbb{R}^{4}, \mathbb{R}^{4})$,
	by Example \ref{diagbasis11},
and since $D_4$ is a proper subgroup of $S_4$ and the spanning set for 
	$\Hom_{D_4}(\mathbb{R}^{4}, \mathbb{R}^{4})$
	includes one further matrix that is not a linear combination of these two matrices, this implies that it must be a basis. 

	Note further, by Example \ref{autocyclegraph},
	that we could have used the cycle graph $C_4$ to obtain automorphism groups that are isomorphic to $D_4$. In fact, if we label the cycle graphs as follows, we obtain
		\begin{center}
			\scalebox{0.9}{\input{compC4.tikz}}
		\end{center}
	which are the complements of the three instances $A, B, C$ of $2K_2$.
	This is clear since $C_4 = \overline{2K_2}$. 
	In fact, in this case, we actually have that $\Aut(C_4) = \Aut(2K_2)$, considered as subgroups of $S_4$, for each of the three possible labellings of the vertices.
	
	However, the basis for
	$\Hom_{D_4}(\mathbb{R}^{4}, \mathbb{R}^{4})$
	for each version of $D_4$ coming from the three instances $\overline{A}, \overline{B}, \overline{C}$ of $C_4$ is different, but related, to each of those for the three instances $A, B, C$ of $2K_2$.
	We present these bases in Figure \ref{D4diagbasis2}.
	Note that only the matrix corresponding to the third $(1,1)$--bilabelled graph diagram $\bm{B_{1,1}}$ has changed, so even though the embeddings of $D_4$ in $S_4$ are the same for a given labelling of the vertices, how those vertices are connected (resulting in either the square $C_4$ or its complement $2K_2$) affects the basis that we obtain for 
	$\Hom_{D_4}(\mathbb{R}^{4}, \mathbb{R}^{4})$.
	In fact, it is clear that, for each of $G = A, B, C$, 
	we have that
	\begin{equation}
		X_{\bm{B_{1,1}}}^{\overline{G}} 
		=
		J - I - 
		X_{\bm{B_{1,1}}}^{G} 
	\end{equation}
	where $J$ is the all ones matrix, and $I$ is the identity matrix.
	By Example \ref{adjacencyex},
	we see that this is, in fact, the equation
	\begin{equation}
		A_{\overline{G}} 
		=
		J - I - 
		A_{G} 
	\end{equation}
	which describes the commonly known result of how to obtain the adjacency matrix of the complement of a graph $G$ from the adjacency matrix of the original graph $G$!

\begin{figure}[tb]
\begin{center}
\begin{tblr}{
		colspec = {X[c]X[c]X[c]X[c]},
  stretch = 0,
  rowsep = 5pt,
  hlines = {1pt},
  vlines = {1pt},
}
	{$(1,1)$--Bilabelled Graph Diagram $\bm{H}$} & 
	{Standard Basis Element $X_{\bm{H}}^{\overline{A}}$ } &
	{Standard Basis Element $X_{\bm{H}}^{\overline{B}}$ } &
	{Standard Basis Element $X_{\bm{H}}^{\overline{C}}$ } \\
	\scalebox{0.6}{\input{graph11sq1.tikz}} & 
	\scalebox{0.85}{
	$
	\NiceMatrixOptions{code-for-first-row = \scriptstyle \color{blue},
                   	   code-for-first-col = \scriptstyle \color{blue}
	}
	\begin{bNiceArray}{*{2}{c}*{2}{c}}[first-row,first-col]
				& 1 	& 2	& 3	& 4 \\
		1		& 1	& 0	& 0	& 0	\\
		2		& 0	& 1	& 0	& 0	\\
		3		& 0	& 0	& 1	& 0	\\
		4		& 0	& 0	& 0	& 1	
	\end{bNiceArray}
	$} &
	\scalebox{0.85}{
	$
	\NiceMatrixOptions{code-for-first-row = \scriptstyle \color{blue},
                   	   code-for-first-col = \scriptstyle \color{blue}
	}
	\begin{bNiceArray}{*{2}{c}*{2}{c}}[first-row,first-col]
				& 1 	& 2	& 3	& 4 \\
		1		& 1	& 0	& 0	& 0	\\
		2		& 0	& 1	& 0	& 0	\\
		3		& 0	& 0	& 1	& 0	\\
		4		& 0	& 0	& 0	& 1	
	\end{bNiceArray}
	$} &
	\scalebox{0.85}{
	$
	\NiceMatrixOptions{code-for-first-row = \scriptstyle \color{blue},
                   	   code-for-first-col = \scriptstyle \color{blue}
	}
	\begin{bNiceArray}{*{2}{c}*{2}{c}}[first-row,first-col]
				& 1 	& 2	& 3	& 4 \\
		1		& 1	& 0	& 0	& 0	\\
		2		& 0	& 1	& 0	& 0	\\
		3		& 0	& 0	& 1	& 0	\\
		4		& 0	& 0	& 0	& 1	
	\end{bNiceArray}
	$}
	\\
	\scalebox{0.6}{\input{graph11sq2.tikz}}	& 
	\scalebox{0.85}{
	$
	\NiceMatrixOptions{code-for-first-row = \scriptstyle \color{blue},
                   	   code-for-first-col = \scriptstyle \color{blue}
	}
	\begin{bNiceArray}{*{2}{c}*{2}{c}}[first-row,first-col]
				& 1 	& 2	& 3	& 4 	\\
		1		& 1	& 1	& 1	& 1	\\
		2		& 1	& 1	& 1	& 1	\\
		3		& 1	& 1	& 1	& 1	\\
		4		& 1	& 1	& 1	& 1	
	\end{bNiceArray}
	$} &
	\scalebox{0.85}{
	$
	\NiceMatrixOptions{code-for-first-row = \scriptstyle \color{blue},
                   	   code-for-first-col = \scriptstyle \color{blue}
	}
	\begin{bNiceArray}{*{2}{c}*{2}{c}}[first-row,first-col]
				& 1 	& 2	& 3	& 4 	\\
		1		& 1	& 1	& 1	& 1	\\
		2		& 1	& 1	& 1	& 1	\\
		3		& 1	& 1	& 1	& 1	\\
		4		& 1	& 1	& 1	& 1	
	\end{bNiceArray}
	$} &
	\scalebox{0.85}{
	$
	\NiceMatrixOptions{code-for-first-row = \scriptstyle \color{blue},
                   	   code-for-first-col = \scriptstyle \color{blue}
	}
	\begin{bNiceArray}{*{2}{c}*{2}{c}}[first-row,first-col]
				& 1 	& 2	& 3	& 4 	\\
		1		& 1	& 1	& 1	& 1	\\
		2		& 1	& 1	& 1	& 1	\\
		3		& 1	& 1	& 1	& 1	\\
		4		& 1	& 1	& 1	& 1	
	\end{bNiceArray}
	$}
	\\
	\scalebox{0.6}{\input{graph11sq3.tikz}}	& 
	\scalebox{0.85}{
	$
	\NiceMatrixOptions{code-for-first-row = \scriptstyle \color{blue},
                   	   code-for-first-col = \scriptstyle \color{blue}
	}
	\begin{bNiceArray}{*{2}{c}*{2}{c}}[first-row,first-col]
				& 1 	& 2	& 3	& 4 	\\
		1		& 0	& 0 	& 1	& 1	\\
		2		& 0	& 0	& 1	& 1	\\
		3		& 1	& 1	& 0	& 0	\\
		4		& 1	& 1	& 0	& 0	
	\end{bNiceArray}
	$} &
	\scalebox{0.85}{
	$
	\NiceMatrixOptions{code-for-first-row = \scriptstyle \color{blue},
                   	   code-for-first-col = \scriptstyle \color{blue}
	}
	\begin{bNiceArray}{*{2}{c}*{2}{c}}[first-row,first-col]
				& 1 	& 2	& 3	& 4 	\\
		1		& 0	& 1	& 0	& 1	\\
		2		& 1	& 0	& 1	& 0	\\
		3		& 0	& 1	& 0	& 1	\\
		4		& 1	& 0	& 1	& 0	
	\end{bNiceArray}
	$} &
	\scalebox{0.85}{
	$
	\NiceMatrixOptions{code-for-first-row = \scriptstyle \color{blue},
                   	   code-for-first-col = \scriptstyle \color{blue}
	}
	\begin{bNiceArray}{*{2}{c}*{2}{c}}[first-row,first-col]
				& 1 	& 2	& 3	& 4 	\\
		1		& 0	& 1	& 1	& 0	\\
		2		& 1	& 0	& 0	& 1	\\
		3		& 1	& 0	& 0	& 1	\\
		4		& 0	& 1	& 1	& 0	
	\end{bNiceArray}
	$}
\end{tblr}
	\caption{
		Choosing the cycle graph $C_4$ instead to obtain an embedding of $D_4$ inside $S_4$,
		we obtain a basis of
		$\Hom_{D_4}(\mathbb{R}^{4}, \mathbb{R}^{4})$, one for each possible labelling of the vertices
		of $C_4$,
		that is related to the basis obtained from the complement graph, $2K_2$.
	}
	\label{D4diagbasis2}
	\end{center}
\end{figure}

\end{example}

\begin{example}[Spanning Set for
$
	\Hom_{S_2}(\mathbb{R}^{3},
	\mathbb{R}^{3})
$]
\label{S2basisR3}

	It is clear that, for a graph $G$ having $n = 3$ vertices such that its
automorphism group is $S_2$,
there are three different ways to label the vertices of the graph,
up to all automorphisms, namely:
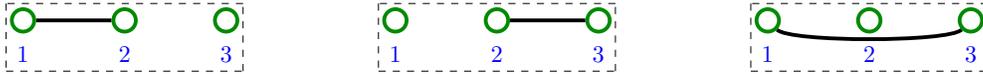
\begin{center}
	\scalebox{0.9}{\input{compS2.tikz}}
\end{center}
We will refer to these graphs as $A$, $B$ and $C$, respectively, throughout the rest of this example.

We focus on graph $A$, with the approach for the other graphs being similar.
In this case, 
$S_2 = \Aut(A) = \{id, (12)\}$ as a subgroup of $S_3$.

Given that the longest path $m$ between any two vertices in $A$ is $1$,
and that $l = k = 1$, we see that, by following the procedure that is given in
the orange box, Steps $1$ to $6$ inclusive are exactly the same as in
Example \ref{D4k1l1ex}.
Hence the $(1,1)$--bilabelled graph diagrams that we need to consider are
those that are given in (\ref{D411i}) and (\ref{D411ii}).

However, given that the graph $A$ is different from the one that appears in Example \ref{D4k1l1ex},
the $A$-homomorphism matrices that correspond to the $(1,1)$--bilabelled graph diagrams
given in (\ref{D411i}) are now
\begin{equation}
	\NiceMatrixOptions{code-for-first-row = \scriptstyle \color{blue},
                   	   code-for-first-col = \scriptstyle \color{blue}
	}
	\begin{bNiceArray}{*{3}{c}}[first-row,first-col]
				& 1 	& 2	& 3	 	\\
		1		& 1	& 0	& 0		\\
		2		& 0	& 1	& 0		\\
		3		& 0	& 0	& 1	
	\end{bNiceArray}
				\; \text{; } \;
	\NiceMatrixOptions{code-for-first-row = \scriptstyle \color{blue},
                   	   code-for-first-col = \scriptstyle \color{blue}
	}
	\begin{bNiceArray}{*{3}{c}}[first-row,first-col]
				& 1 	& 2	& 3	 	\\
		1		& 1	& 1	& 1		\\
		2		& 1	& 1	& 1		\\
		3		& 1	& 1	& 1	
	\end{bNiceArray}
				\; \text{; } \;
	\NiceMatrixOptions{code-for-first-row = \scriptstyle \color{blue},
                   	   code-for-first-col = \scriptstyle \color{blue}
	}
	\begin{bNiceArray}{*{3}{c}}[first-row,first-col]
				& 1 	& 2	& 3	 	\\
		1		& 0	& 1	& 0		\\
		2		& 1	& 0	& 0		\\
		3		& 0	& 0	& 0	
	\end{bNiceArray}
				\; \text{and } \;
		\NiceMatrixOptions{code-for-first-row = \scriptstyle \color{blue},
                   	   code-for-first-col = \scriptstyle \color{blue}
	}
	\begin{bNiceArray}{*{3}{c}}[first-row,first-col]
				& 1 	& 2	& 3	 	\\
		1		& 1	& 0	& 0		\\
		2		& 0	& 1	& 0		\\
		3		& 0	& 0	& 0	
	\end{bNiceArray}
\end{equation}
and the $A$-homomorphism matrices that correspond to the $(1,1)$--bilabelled graph diagrams given in (\ref{D411ii}) are now
\begin{equation}
	\NiceMatrixOptions{code-for-first-row = \scriptstyle \color{blue},
                   	   code-for-first-col = \scriptstyle \color{blue}
	}
	\begin{bNiceArray}{*{3}{c}}[first-row,first-col]
				& 1 	& 2	& 3	 	\\
		1		& 1	& 0	& 0		\\
		2		& 0	& 1	& 0		\\
		3		& 0	& 0	& 0	
	\end{bNiceArray}
				\; \text{; } \;
	\NiceMatrixOptions{code-for-first-row = \scriptstyle \color{blue},
                   	   code-for-first-col = \scriptstyle \color{blue}
	}
	\begin{bNiceArray}{*{3}{c}}[first-row,first-col]
				& 1 	& 2	& 3	 	\\
		1		& 1	& 1	& 0		\\
		2		& 1	& 1	& 0		\\
		3		& 1	& 1	& 0	
	\end{bNiceArray}
				\; \text{; } \;
	\NiceMatrixOptions{code-for-first-row = \scriptstyle \color{blue},
                   	   code-for-first-col = \scriptstyle \color{blue}
	}
	\begin{bNiceArray}{*{3}{c}}[first-row,first-col]
				& 1 	& 2	& 3	 	\\
		1		& 1	& 1	& 1		\\
		2		& 1	& 1	& 1		\\
		3		& 0	& 0	& 0	
	\end{bNiceArray}
				\; \text{and } \;
	\NiceMatrixOptions{code-for-first-row = \scriptstyle \color{blue},
                   	   code-for-first-col = \scriptstyle \color{blue}
	}
	\begin{bNiceArray}{*{3}{c}}[first-row,first-col]
				& 1 	& 2	& 3	 	\\
		1		& 1	& 1	& 0		\\
		2		& 1	& 1	& 0		\\
		3		& 0	& 0	& 0	
	\end{bNiceArray}
\end{equation}

All of the $A$-homomorphism matrices are unique except for those coming from
$\bm{B_{1,2}}$ and $\bm{A_2}$, and so by removing the matrix corresponding to $\bm{A_2}$,
we obtain a spanning set
$
	\Hom_{S_2}(\mathbb{R}^{3},
	\mathbb{R}^{3})
$,
for $S_2 = \Aut(A) = \langle (id, (12) \rangle$, as a subgroup of $S_3$.
We label the remaining matrices in the order in which they are presented
using the index set $\{1, \dots, 7\}$.

Hence the weight matrix that we obtain for an $\Aut(A) \cong S_2$-equivariant linear
layer function from 
$\mathbb{R}^3$ to 
$\mathbb{R}^3$ is
\begin{equation}
	\NiceMatrixOptions{code-for-first-row = \scriptstyle \color{blue},
                   	   code-for-first-col = \scriptstyle \color{blue}
	}
	\begin{bNiceArray}{*{3}{c}}[first-row,first-col]
				& 1 	& 2	& 3	\\
		1		& \lambda_{1,2,4,5,6,7}	& \lambda_{2,3,5,6,7}	& \lambda_{2,6}	\\
		2		& \lambda_{2,3,5,6,7}	& \lambda_{1,2,4,5,6,7}	& \lambda_{2,6}	\\
		3		& \lambda_{2,5}	& \lambda_{2,5}	& \lambda_{1,2}	
	\end{bNiceArray}
\end{equation}
for weights $\lambda_1, \dots, \lambda_7 \in \mathbb{R}$. 

We can find the elements of the spanning set for
$
	\Hom_{S_2}(\mathbb{R}^{3},
	\mathbb{R}^{3})
$
for the other two embeddings of $S_2$ in $S_3$, given by $\Aut(B)$ and $\Aut(C)$.
For $\Aut(B)$, we perform the permutation $(13)$ on each index of the rows and columns 
of the spanning set matrices found for $\Aut(A)$, and for $\Aut(C)$,
we perform the permutation $(23)$ instead.

\end{example}

\begin{example}[Spanning Set for
$
	\Hom_{S_2}((\mathbb{R}^{3})^{\otimes 2},
	\mathbb{R}^{3})
$]

	We refer to the same three graphs, $A, B, C$, that were given in Example \ref{S2basisR3},
and, once again, focus on graph $A$, with the approach for the other graphs being similar.

Recall that, in this case, $S_2 = \Aut(A) = \{id, (12)\}$, as a subgroup of $S_3$. 
We still have that the longest path in $A$ is $1$.
We follow the procedure that is given in the orange box.
However, in order to keep the $A$-homomorphism matrices 
in close proximity with
the bilabelled graph diagrams that they correspond to, 
we have chosen to apply Step 6 and the first part of Step 7 during Steps 1 through 5
inclusive.

\textbf{Step 1:} We calculate all set partitions of $[l+k] = \{1,2,3\}$ and
express them as $(3,0)$--bilabelled graph diagrams, labelling only the black vertices.

They are given by

\begin{equation} \label{S2setparts21}
	\begin{aligned}
		\bm{A_0} = \scalebox{0.6}{\input{graph21sq1line.tikz}}
		\; \text{;} \;
		\bm{B_0} = \scalebox{0.6}{\input{graph21sq2line.tikz}}
		\; \text{;} \;
		\bm{C_0} = \scalebox{0.6}{\input{graph21sq3line.tikz}}
		\; \text{;} \;
		\bm{D_0} = \scalebox{0.6}{\input{graph21sq4line.tikz}}
		\; \text{and} \;
		\bm{E_0} = \scalebox{0.6}{\input{graph21sq5line.tikz}}
	\end{aligned}
\end{equation}

and correspond, after Frobenius duality, to the $A$-homomorphism matrices

\begin{equation}
	\scalebox{0.75}{
	$
	\NiceMatrixOptions{code-for-first-row = \scriptstyle \color{blue},
                   	   code-for-first-col = \scriptstyle \color{blue}
	}
	\begin{bNiceArray}{*{3}{c}|*{3}{c}|*{3}{c}}[first-row,first-col]
				& 1,1 	& 1,2	& 1,3	& 2,1 	& 2,2	& 2,3	& 3,1 	& 3,2	& 3,3	\\
		1		& 1	& 0	& 0	& 0	& 0	& 0	& 0	& 0	& 0	\\
		2		& 0	& 0	& 0	& 0	& 1	& 0	& 0	& 0	& 0	\\
		3		& 0	& 0	& 0	& 0	& 0	& 0	& 0	& 0	& 1
	\end{bNiceArray}
	$}
	\; \text{;} \;
	\scalebox{0.75}{
	$
	\NiceMatrixOptions{code-for-first-row = \scriptstyle \color{blue},
                   	   code-for-first-col = \scriptstyle \color{blue}
	}
	\begin{bNiceArray}{*{3}{c}|*{3}{c}|*{3}{c}}[first-row,first-col]
				& 1,1 	& 1,2	& 1,3	& 2,1 	& 2,2	& 2,3	& 3,1 	& 3,2	& 3,3	\\
		1		& 1	& 1	& 1	& 0	& 0	& 0	& 0	& 0	& 0	\\
		2		& 0	& 0	& 0	& 1	& 1	& 1	& 0	& 0	& 0	\\
		3		& 0	& 0	& 0	& 0	& 0	& 0	& 1	& 1	& 1
	\end{bNiceArray}
	$}
	\; \text{;} \;
	\scalebox{0.75}{
	$
	\NiceMatrixOptions{code-for-first-row = \scriptstyle \color{blue},
                   	   code-for-first-col = \scriptstyle \color{blue}
	}
	\begin{bNiceArray}{*{3}{c}|*{3}{c}|*{3}{c}}[first-row,first-col]
				& 1,1 	& 1,2	& 1,3	& 2,1 	& 2,2	& 2,3	& 3,1 	& 3,2	& 3,3	\\
		1		& 1	& 0	& 0	& 1	& 0	& 0	& 1	& 0	& 0	\\
		2		& 0	& 1	& 0	& 0	& 1	& 0	& 0	& 1	& 0	\\
		3		& 0	& 0	& 1	& 0	& 0	& 1	& 0	& 0	& 1
	\end{bNiceArray}
	$}
\end{equation}
\begin{equation}
	\scalebox{0.75}{
	$
	\NiceMatrixOptions{code-for-first-row = \scriptstyle \color{blue},
                   	   code-for-first-col = \scriptstyle \color{blue}
	}
	\begin{bNiceArray}{*{3}{c}|*{3}{c}|*{3}{c}}[first-row,first-col]
				& 1,1 	& 1,2	& 1,3	& 2,1 	& 2,2	& 2,3	& 3,1 	& 3,2	& 3,3	\\
		1		& 1	& 0	& 0	& 0	& 1	& 0	& 0	& 0	& 1	\\
		2		& 1	& 0	& 0	& 0	& 1	& 0	& 0	& 0	& 1	\\
		3		& 1	& 0	& 0	& 0	& 1	& 0	& 0	& 0	& 1
	\end{bNiceArray}
	$}
	\; \text{and} \;
		\scalebox{0.75}{
	$
	\NiceMatrixOptions{code-for-first-row = \scriptstyle \color{blue},
                   	   code-for-first-col = \scriptstyle \color{blue}
	}
	\begin{bNiceArray}{*{3}{c}|*{3}{c}|*{3}{c}}[first-row,first-col]
				& 1,1 	& 1,2	& 1,3	& 2,1 	& 2,2	& 2,3	& 3,1 	& 3,2	& 3,3	\\
		1		& 1	& 1	& 1	& 1	& 1	& 1	& 1	& 1	& 1	\\
		2		& 1	& 1	& 1	& 1	& 1	& 1	& 1	& 1	& 1	\\
		3		& 1	& 1	& 1	& 1	& 1	& 1	& 1	& 1	& 1
	\end{bNiceArray}
	$}
\end{equation}

\textbf{Step 2:} We calculate all possible $(3,0)$--bilabelled graph diagrams that
have only internal red edges between red vertices from the five bilabelled graph diagrams 
given in Step 1.

Clearly there are no new bilabelled graph diagrams coming from $\bm{A_0}$ in this step.

However, for $\bm{B_0}$, $\bm{C_0}$ and $\bm{D_0}$, calling these generically by $\bm{H}$,
we have that
\begin{itemize}
	\item The number of red vertices, $c$, in $\bm{H}$ is $2$. Hence the number of edges in the complete graph $K_2$, $e$, is $1$. 
	\item We now create all possible length $e=1$ strings having values in $0 \rightarrow 2m = 2$. Remove the all $0$ string. 
\end{itemize}
Hence we obtain the length one strings $1$ and $2$.
We use these strings to create two new $(3,0)$--bilabelled graph diagrams from each of
$\bm{B_0}$, $\bm{C_0}$ and $\bm{D_0}$,
by inserting $t$ new edges between the two red vertices,
where $t$ equals $1$ and $2$, respectively.
The new $(3,0)$--bilabelled graph diagrams are given as follows:

\begin{equation} \label{S2internaledgesi}
	\begin{aligned}
		\bm{B_{1,1}} = \scalebox{0.6}{\input{graph21sq6line.tikz}}
		\; \text{;} \;
		\bm{B_{1,2}} = \scalebox{0.6}{\input{graph21sq7line.tikz}}
		\; \text{;} \;
		\bm{C_{1,1}} = \scalebox{0.6}{\input{graph21sq8line.tikz}}
	\end{aligned}
\end{equation}

\begin{equation} \label{S2internaledgesii}
	\begin{aligned}
		\bm{C_{1,2}} = \scalebox{0.6}{\input{graph21sq9line.tikz}}
		\; \text{;} \;
		\bm{D_{1,1}} = \scalebox{0.6}{\input{graph21sq10line.tikz}}
		\; \text{;} \;
		\bm{D_{1,2}} = \scalebox{0.6}{\input{graph21sq11line.tikz}}
	\end{aligned}
\end{equation}

They correspond, after Frobenius duality, to the $A$-homomorphism matrices

\begin{equation}
	\scalebox{0.75}{
	$
	\NiceMatrixOptions{code-for-first-row = \scriptstyle \color{blue},
                   	   code-for-first-col = \scriptstyle \color{blue}
	}
	\begin{bNiceArray}{*{3}{c}|*{3}{c}|*{3}{c}}[first-row,first-col]
				& 1,1 	& 1,2	& 1,3	& 2,1 	& 2,2	& 2,3	& 3,1 	& 3,2	& 3,3	\\
		1		& 0	& 1	& 0	& 0	& 0	& 0	& 0	& 0	& 0	\\
		2		& 0	& 0	& 0	& 1	& 0	& 0	& 0	& 0	& 0	\\
		3		& 0	& 0	& 0	& 0	& 0	& 0	& 0	& 0	& 0
	\end{bNiceArray}
	$}
	\; \text{;} \;
		\scalebox{0.75}{
	$
	\NiceMatrixOptions{code-for-first-row = \scriptstyle \color{blue},
                   	   code-for-first-col = \scriptstyle \color{blue}
	}
	\begin{bNiceArray}{*{3}{c}|*{3}{c}|*{3}{c}}[first-row,first-col]
				& 1,1 	& 1,2	& 1,3	& 2,1 	& 2,2	& 2,3	& 3,1 	& 3,2	& 3,3	\\
		1		& 1	& 0	& 0	& 0	& 0	& 0	& 0	& 0	& 0	\\
		2		& 0	& 0	& 0	& 0	& 1	& 0	& 0	& 0	& 0	\\
		3		& 0	& 0	& 0	& 0	& 0	& 0	& 0	& 0	& 0
	\end{bNiceArray}
	$}
	\; \text{;} \;
		\scalebox{0.75}{
	$
	\NiceMatrixOptions{code-for-first-row = \scriptstyle \color{blue},
                   	   code-for-first-col = \scriptstyle \color{blue}
	}
	\begin{bNiceArray}{*{3}{c}|*{3}{c}|*{3}{c}}[first-row,first-col]
				& 1,1 	& 1,2	& 1,3	& 2,1 	& 2,2	& 2,3	& 3,1 	& 3,2	& 3,3	\\
		1		& 0	& 0	& 0	& 1	& 0	& 0	& 0	& 0	& 0	\\
		2		& 0	& 1	& 0	& 0	& 0	& 0	& 0	& 0	& 0	\\
		3		& 0	& 0	& 0	& 0	& 0	& 0	& 0	& 0	& 0
	\end{bNiceArray}
	$}
\end{equation}
\begin{equation}
	\scalebox{0.75}{
	$
	\NiceMatrixOptions{code-for-first-row = \scriptstyle \color{blue},
                   	   code-for-first-col = \scriptstyle \color{blue}
	}
	\begin{bNiceArray}{*{3}{c}|*{3}{c}|*{3}{c}}[first-row,first-col]
				& 1,1 	& 1,2	& 1,3	& 2,1 	& 2,2	& 2,3	& 3,1 	& 3,2	& 3,3	\\
		1		& 1	& 0	& 0	& 0	& 0	& 0	& 0	& 0	& 0	\\
		2		& 0	& 0	& 0	& 0	& 1	& 0	& 0	& 0	& 0	\\
		3		& 0	& 0	& 0	& 0	& 0	& 0	& 0	& 0	& 0
	\end{bNiceArray}
	$}
	\; \text{;} \;
	\scalebox{0.75}{
	$
	\NiceMatrixOptions{code-for-first-row = \scriptstyle \color{blue},
                   	   code-for-first-col = \scriptstyle \color{blue}
	}
	\begin{bNiceArray}{*{3}{c}|*{3}{c}|*{3}{c}}[first-row,first-col]
				& 1,1 	& 1,2	& 1,3	& 2,1 	& 2,2	& 2,3	& 3,1 	& 3,2	& 3,3	\\
		1		& 0	& 0	& 0	& 0	& 1	& 0	& 0	& 0	& 0	\\
		2		& 1	& 0	& 0	& 0	& 0	& 0	& 0	& 0	& 0	\\
		3		& 0	& 0	& 0	& 0	& 0	& 0	& 0	& 0	& 0
	\end{bNiceArray}
	$}
	\; \text{;} \;
	\scalebox{0.75}{
	$
	\NiceMatrixOptions{code-for-first-row = \scriptstyle \color{blue},
                   	   code-for-first-col = \scriptstyle \color{blue}
	}
	\begin{bNiceArray}{*{3}{c}|*{3}{c}|*{3}{c}}[first-row,first-col]
				& 1,1 	& 1,2	& 1,3	& 2,1 	& 2,2	& 2,3	& 3,1 	& 3,2	& 3,3	\\
		1		& 1	& 0	& 0	& 0	& 0	& 0	& 0	& 0	& 0	\\
		2		& 0	& 0	& 0	& 0	& 1	& 0	& 0	& 0	& 0	\\
		3		& 0	& 0	& 0	& 0	& 0	& 0	& 0	& 0	& 0
	\end{bNiceArray}
	$}
\end{equation}

For $\bm{E_0}$, we have that
\begin{itemize}
	\item The number of red vertices, $c$, in $\bm{E_0}$ is $3$. Hence the number of edges in the complete graph $K_3$, $e$, is $3$. 
	\item We now create all possible length $e=3$ strings having values in $0 \rightarrow 2m = 2$. Remove the all $0$ string. 
\end{itemize}
There are $26$ such strings. If we label the red vertices in $\bm{E_0}$,
without loss of generality,
as $1$, $2$ and $3$, from left to right, then we can
let the indices in each string refer to possible connections between
vertices $1$ and $2$, vertices $1$ and $3$, and vertices $2$ and $3$, written
$(12)$, $(13)$ and $(23)$, also from left to right and
without loss of generality.

\begin{table}[htbp]
\centering
\caption{All Possible Length $3$ Strings With Entries in $0 \rightarrow 2m = 2$}
\label{tab:length3strings}
\begin{tabular}{|c|c|c|}
\hline
	\textcolor{red}{(12)(13)(23)} & \textcolor{red}{(12)(13)(23)} & \textcolor{red}{(12)(13)(23)} \\
\hline
$000$ & $001$ & $002$ \\
$010$ & $011$ & $012$ \\
$020$ & $021$ & $022$ \\
$100$ & $101$ & $102$ \\
$110$ & $111$ & $112$ \\
$120$ & $121$ & $122$ \\
$200$ & $201$ & $202$ \\
$210$ & $211$ & $212$ \\
$220$ & $221$ & $222$ \\
\hline
\end{tabular}
\end{table}

To create the $26$ new $(3,0)$--bilabelled graph diagrams,
we now add in the edges (and new red vertices) between the labelled 
red vertices in $\bm{E_0}$, according
to the entries in the string.

They are

\begin{equation} \label{S2internaledges000}
	\begin{aligned}
		\bm{E_{001}} = \scalebox{0.6}{\input{graph21Internal001.tikz}}
		\; \text{;} \;
		\bm{E_{002}} = \scalebox{0.6}{\input{graph21Internal002.tikz}}
	\end{aligned}
\end{equation}

\begin{equation} \label{S2internaledges010}
	\begin{aligned}
		\bm{E_{010}} = \scalebox{0.6}{\input{graph21Internal010.tikz}}
		\; \text{;} \;
		\bm{E_{011}} = \scalebox{0.6}{\input{graph21Internal011.tikz}}
		\; \text{;} \;
		\bm{E_{012}} = \scalebox{0.6}{\input{graph21Internal012.tikz}}
	\end{aligned}
\end{equation}

\begin{equation} \label{S2internaledges020}
	\begin{aligned}
		\bm{E_{020}} = \scalebox{0.6}{\input{graph21Internal020.tikz}}
		\; \text{;} \;
		\bm{E_{021}} = \scalebox{0.6}{\input{graph21Internal021.tikz}}
		\; \text{;} \;
		\bm{E_{022}} = \scalebox{0.6}{\input{graph21Internal022.tikz}}
	\end{aligned}
\end{equation}

\begin{equation} \label{S2internaledges100}
	\begin{aligned}
		\bm{E_{100}} = \scalebox{0.6}{\input{graph21Internal100.tikz}}
		\; \text{;} \;
		\bm{E_{101}} = \scalebox{0.6}{\input{graph21Internal101.tikz}}
		\; \text{;} \;
		\bm{E_{102}} = \scalebox{0.6}{\input{graph21Internal102.tikz}}
	\end{aligned}
\end{equation}

\begin{equation} \label{S2internaledges110}
	\begin{aligned}
		\bm{E_{110}} = \scalebox{0.6}{\input{graph21Internal110.tikz}}
		\; \text{;} \;
		\bm{E_{111}} = \scalebox{0.6}{\input{graph21Internal111.tikz}}
		\; \text{;} \;
		\bm{E_{112}} = \scalebox{0.6}{\input{graph21Internal112.tikz}}
	\end{aligned}
\end{equation}

\begin{equation} \label{S2internaledges120}
	\begin{aligned}
		\bm{E_{120}} = \scalebox{0.6}{\input{graph21Internal120.tikz}}
		\; \text{;} \;
		\bm{E_{121}} = \scalebox{0.6}{\input{graph21Internal121.tikz}}
		\; \text{;} \;
		\bm{E_{122}} = \scalebox{0.6}{\input{graph21Internal122.tikz}}
	\end{aligned}
\end{equation}

\begin{equation} \label{S2internaledges200}
	\begin{aligned}
		\bm{E_{200}} = \scalebox{0.6}{\input{graph21Internal200.tikz}}
		\; \text{;} \;
		\bm{E_{201}} = \scalebox{0.6}{\input{graph21Internal201.tikz}}
		\; \text{;} \;
		\bm{E_{202}} = \scalebox{0.6}{\input{graph21Internal202.tikz}}
	\end{aligned}
\end{equation}

\begin{equation} \label{S2internaledges110}
	\begin{aligned}
		\bm{E_{210}} = \scalebox{0.6}{\input{graph21Internal210.tikz}}
		\; \text{;} \;
		\bm{E_{211}} = \scalebox{0.6}{\input{graph21Internal211.tikz}}
		\; \text{;} \;
		\bm{E_{212}} = \scalebox{0.6}{\input{graph21Internal212.tikz}}
	\end{aligned}
\end{equation}

\begin{equation} \label{S2internaledges120}
	\begin{aligned}
		\bm{E_{220}} = \scalebox{0.6}{\input{graph21Internal220.tikz}}
		\; \text{;} \;
		\bm{E_{221}} = \scalebox{0.6}{\input{graph21Internal221.tikz}}
		\; \text{;} \;
		\bm{E_{222}} = \scalebox{0.6}{\input{graph21Internal222.tikz}}
	\end{aligned}
\end{equation}

and correspond, after Frobenius duality, to the $A$-homomorphism matrices

\begin{equation}
	\scalebox{0.75}{
	$
	\NiceMatrixOptions{code-for-first-row = \scriptstyle \color{blue},
                   	   code-for-first-col = \scriptstyle \color{blue}
	}
	\begin{bNiceArray}{*{3}{c}|*{3}{c}|*{3}{c}}[first-row,first-col]
				& 1,1 	& 1,2	& 1,3	& 2,1 	& 2,2	& 2,3	& 3,1 	& 3,2	& 3,3	\\
		1		& 0	& 1	& 0	& 1	& 0	& 0	& 0	& 0	& 0	\\
		2		& 0	& 1	& 0	& 1	& 0	& 0	& 0	& 0	& 0	\\
		3		& 0	& 1	& 0	& 1	& 0	& 0	& 0	& 0	& 0
	\end{bNiceArray}
	$}
	\; \text{;} \;
	\scalebox{0.75}{
	$
	\NiceMatrixOptions{code-for-first-row = \scriptstyle \color{blue},
                   	   code-for-first-col = \scriptstyle \color{blue}
	}
	\begin{bNiceArray}{*{3}{c}|*{3}{c}|*{3}{c}}[first-row,first-col]
				& 1,1 	& 1,2	& 1,3	& 2,1 	& 2,2	& 2,3	& 3,1 	& 3,2	& 3,3	\\
		1		& 1	& 0	& 0	& 0	& 1	& 0	& 0	& 0	& 0	\\
		2		& 1	& 0	& 0	& 0	& 1	& 0	& 0	& 0	& 0	\\
		3		& 1	& 0	& 0	& 0	& 1	& 0	& 0	& 0	& 0
	\end{bNiceArray}
	$}
\end{equation}
\begin{equation}
	\scalebox{0.75}{
	$
	\NiceMatrixOptions{code-for-first-row = \scriptstyle \color{blue},
                   	   code-for-first-col = \scriptstyle \color{blue}
	}
	\begin{bNiceArray}{*{3}{c}|*{3}{c}|*{3}{c}}[first-row,first-col]
				& 1,1 	& 1,2	& 1,3	& 2,1 	& 2,2	& 2,3	& 3,1 	& 3,2	& 3,3	\\
		1		& 0	& 1	& 0	& 0	& 1	& 0	& 0	& 1	& 0	\\
		2		& 1	& 0	& 0	& 1	& 0	& 0	& 1	& 0	& 0	\\
		3		& 0	& 0	& 0	& 0	& 0	& 0	& 0	& 0	& 0
	\end{bNiceArray}
	$}
	\; \text{;} \;
	\scalebox{0.75}{
	$
	\NiceMatrixOptions{code-for-first-row = \scriptstyle \color{blue},
                   	   code-for-first-col = \scriptstyle \color{blue}
	}
	\begin{bNiceArray}{*{3}{c}|*{3}{c}|*{3}{c}}[first-row,first-col]
				& 1,1 	& 1,2	& 1,3	& 2,1 	& 2,2	& 2,3	& 3,1 	& 3,2	& 3,3	\\
		1		& 0	& 1	& 0	& 0	& 1	& 0	& 0	& 0	& 0	\\
		2		& 1	& 0	& 0	& 1	& 0	& 0	& 0	& 0	& 0	\\
		3		& 0	& 0	& 0	& 0	& 0	& 0	& 0	& 0	& 0
	\end{bNiceArray}
	$}
	\; \text{;} \;
	\scalebox{0.75}{
	$
	\NiceMatrixOptions{code-for-first-row = \scriptstyle \color{blue},
                   	   code-for-first-col = \scriptstyle \color{blue}
	}
	\begin{bNiceArray}{*{3}{c}|*{3}{c}|*{3}{c}}[first-row,first-col]
				& 1,1 	& 1,2	& 1,3	& 2,1 	& 2,2	& 2,3	& 3,1 	& 3,2	& 3,3	\\
		1		& 0	& 0	& 0	& 0	& 1	& 0	& 0	& 0	& 0	\\
		2		& 1	& 0	& 0	& 0	& 0	& 0	& 0	& 0	& 0	\\
		3		& 0	& 0	& 0	& 0	& 0	& 0	& 0	& 0	& 0
	\end{bNiceArray}
	$}
\end{equation}
\begin{equation}
	\scalebox{0.75}{
	$
	\NiceMatrixOptions{code-for-first-row = \scriptstyle \color{blue},
                   	   code-for-first-col = \scriptstyle \color{blue}
	}
	\begin{bNiceArray}{*{3}{c}|*{3}{c}|*{3}{c}}[first-row,first-col]
				& 1,1 	& 1,2	& 1,3	& 2,1 	& 2,2	& 2,3	& 3,1 	& 3,2	& 3,3	\\
		1		& 1	& 0	& 0	& 1	& 0	& 0	& 1	& 0	& 0	\\
		2		& 0	& 1	& 0	& 0	& 1	& 0	& 0	& 1	& 0	\\
		3		& 0	& 0	& 0	& 0	& 0	& 0	& 0	& 0	& 0
	\end{bNiceArray}
	$}
	\; \text{;} \;
	\scalebox{0.75}{
	$
	\NiceMatrixOptions{code-for-first-row = \scriptstyle \color{blue},
                   	   code-for-first-col = \scriptstyle \color{blue}
	}
	\begin{bNiceArray}{*{3}{c}|*{3}{c}|*{3}{c}}[first-row,first-col]
				& 1,1 	& 1,2	& 1,3	& 2,1 	& 2,2	& 2,3	& 3,1 	& 3,2	& 3,3	\\
		1		& 0	& 0	& 0	& 1	& 0	& 0	& 0	& 0	& 0	\\
		2		& 0	& 1	& 0	& 0	& 0	& 0	& 0	& 0	& 0	\\
		3		& 0	& 0	& 0	& 0	& 0	& 0	& 0	& 0	& 0
	\end{bNiceArray}
	$}
	\; \text{;} \;
	\scalebox{0.75}{
	$
	\NiceMatrixOptions{code-for-first-row = \scriptstyle \color{blue},
                   	   code-for-first-col = \scriptstyle \color{blue}
	}
	\begin{bNiceArray}{*{3}{c}|*{3}{c}|*{3}{c}}[first-row,first-col]
				& 1,1 	& 1,2	& 1,3	& 2,1 	& 2,2	& 2,3	& 3,1 	& 3,2	& 3,3	\\
		1		& 1	& 0	& 0	& 0	& 0	& 0	& 0	& 0	& 0	\\
		2		& 0	& 0	& 0	& 0	& 1	& 0	& 0	& 0	& 0	\\
		3		& 0	& 0	& 0	& 0	& 0	& 0	& 0	& 0	& 0
	\end{bNiceArray}
	$}
\end{equation}
\begin{equation}
	\scalebox{0.75}{
	$
	\NiceMatrixOptions{code-for-first-row = \scriptstyle \color{blue},
                   	   code-for-first-col = \scriptstyle \color{blue}
	}
	\begin{bNiceArray}{*{3}{c}|*{3}{c}|*{3}{c}}[first-row,first-col]
				& 1,1 	& 1,2	& 1,3	& 2,1 	& 2,2	& 2,3	& 3,1 	& 3,2	& 3,3	\\
		1		& 0	& 0	& 0	& 1	& 1	& 1	& 0	& 0	& 0	\\
		2		& 1	& 1	& 1	& 0	& 0	& 0	& 0	& 0	& 0	\\
		3		& 0	& 0	& 0	& 0	& 0	& 0	& 0	& 0	& 0
	\end{bNiceArray}
	$}
	\; \text{;} \;
	\scalebox{0.75}{
	$
	\NiceMatrixOptions{code-for-first-row = \scriptstyle \color{blue},
                   	   code-for-first-col = \scriptstyle \color{blue}
	}
	\begin{bNiceArray}{*{3}{c}|*{3}{c}|*{3}{c}}[first-row,first-col]
				& 1,1 	& 1,2	& 1,3	& 2,1 	& 2,2	& 2,3	& 3,1 	& 3,2	& 3,3	\\
		1		& 0	& 0	& 0	& 1	& 0	& 0	& 0	& 0	& 0	\\
		2		& 0	& 1	& 0	& 0	& 0	& 0	& 0	& 0	& 0	\\
		3		& 0	& 0	& 0	& 0	& 0	& 0	& 0	& 0	& 0
	\end{bNiceArray}
	$}
	\; \text{;} \;
	\scalebox{0.75}{
	$
	\NiceMatrixOptions{code-for-first-row = \scriptstyle \color{blue},
                   	   code-for-first-col = \scriptstyle \color{blue}
	}
	\begin{bNiceArray}{*{3}{c}|*{3}{c}|*{3}{c}}[first-row,first-col]
				& 1,1 	& 1,2	& 1,3	& 2,1 	& 2,2	& 2,3	& 3,1 	& 3,2	& 3,3	\\
		1		& 0	& 0	& 0	& 0	& 1	& 0	& 0	& 0	& 0	\\
		2		& 1	& 0	& 0	& 0	& 0	& 0	& 0	& 0	& 0	\\
		3		& 0	& 0	& 0	& 0	& 0	& 0	& 0	& 0	& 0
	\end{bNiceArray}
	$}
\end{equation}
\begin{equation}
	\scalebox{0.75}{
	$
	\NiceMatrixOptions{code-for-first-row = \scriptstyle \color{blue},
                   	   code-for-first-col = \scriptstyle \color{blue}
	}
	\begin{bNiceArray}{*{3}{c}|*{3}{c}|*{3}{c}}[first-row,first-col]
				& 1,1 	& 1,2	& 1,3	& 2,1 	& 2,2	& 2,3	& 3,1 	& 3,2	& 3,3	\\
		1		& 0	& 0	& 0	& 0	& 1	& 0	& 0	& 0	& 0	\\
		2		& 1	& 0	& 0	& 0	& 0	& 0	& 0	& 0	& 0	\\
		3		& 0	& 0	& 0	& 0	& 0	& 0	& 0	& 0	& 0
	\end{bNiceArray}
	$}
	\; \text{;} \;
	\scalebox{0.75}{
	$
	\NiceMatrixOptions{code-for-first-row = \scriptstyle \color{blue},
                   	   code-for-first-col = \scriptstyle \color{blue}
	}
	\begin{bNiceArray}{*{3}{c}|*{3}{c}|*{3}{c}}[first-row,first-col]
				& 1,1 	& 1,2	& 1,3	& 2,1 	& 2,2	& 2,3	& 3,1 	& 3,2	& 3,3	\\
		1		& 0	& 0	& 0	& 0	& 0	& 0	& 0	& 0	& 0	\\
		2		& 0	& 0	& 0	& 0	& 0	& 0	& 0	& 0	& 0	\\
		3		& 0	& 0	& 0	& 0	& 0	& 0	& 0	& 0	& 0
	\end{bNiceArray}
	$}
	\; \text{;} \;
	\scalebox{0.75}{
	$
	\NiceMatrixOptions{code-for-first-row = \scriptstyle \color{blue},
                   	   code-for-first-col = \scriptstyle \color{blue}
	}
	\begin{bNiceArray}{*{3}{c}|*{3}{c}|*{3}{c}}[first-row,first-col]
				& 1,1 	& 1,2	& 1,3	& 2,1 	& 2,2	& 2,3	& 3,1 	& 3,2	& 3,3	\\
		1		& 0	& 0	& 0	& 0	& 1	& 0	& 0	& 0	& 0	\\
		2		& 1	& 0	& 0	& 0	& 0	& 0	& 0	& 0	& 0	\\
		3		& 0	& 0	& 0	& 0	& 0	& 0	& 0	& 0	& 0
	\end{bNiceArray}
	$}
\end{equation}
\begin{equation}
	\scalebox{0.75}{
	$
	\NiceMatrixOptions{code-for-first-row = \scriptstyle \color{blue},
                   	   code-for-first-col = \scriptstyle \color{blue}
	}
	\begin{bNiceArray}{*{3}{c}|*{3}{c}|*{3}{c}}[first-row,first-col]
				& 1,1 	& 1,2	& 1,3	& 2,1 	& 2,2	& 2,3	& 3,1 	& 3,2	& 3,3	\\
		1		& 0	& 0	& 0	& 1	& 0	& 0	& 0	& 0	& 0	\\
		2		& 0	& 1	& 0	& 0	& 0	& 0	& 0	& 0	& 0	\\
		3		& 0	& 0	& 0	& 0	& 0	& 0	& 0	& 0	& 0
	\end{bNiceArray}
	$}
	\; \text{;} \;
	\scalebox{0.75}{
	$
	\NiceMatrixOptions{code-for-first-row = \scriptstyle \color{blue},
                   	   code-for-first-col = \scriptstyle \color{blue}
	}
	\begin{bNiceArray}{*{3}{c}|*{3}{c}|*{3}{c}}[first-row,first-col]
				& 1,1 	& 1,2	& 1,3	& 2,1 	& 2,2	& 2,3	& 3,1 	& 3,2	& 3,3	\\
		1		& 0	& 0	& 0	& 1	& 0	& 0	& 0	& 0	& 0	\\
		2		& 0	& 1	& 0	& 0	& 0	& 0	& 0	& 0	& 0	\\
		3		& 0	& 0	& 0	& 0	& 0	& 0	& 0	& 0	& 0
	\end{bNiceArray}
	$}
	\; \text{;} \;
	\scalebox{0.75}{
	$
	\NiceMatrixOptions{code-for-first-row = \scriptstyle \color{blue},
                   	   code-for-first-col = \scriptstyle \color{blue}
	}
	\begin{bNiceArray}{*{3}{c}|*{3}{c}|*{3}{c}}[first-row,first-col]
				& 1,1 	& 1,2	& 1,3	& 2,1 	& 2,2	& 2,3	& 3,1 	& 3,2	& 3,3	\\
		1		& 0	& 0	& 0	& 0	& 0	& 0	& 0	& 0	& 0	\\
		2		& 0	& 0	& 0	& 0	& 0	& 0	& 0	& 0	& 0	\\
		3		& 0	& 0	& 0	& 0	& 0	& 0	& 0	& 0	& 0
	\end{bNiceArray}
	$}
\end{equation}
\begin{equation}
	\scalebox{0.75}{
	$
	\NiceMatrixOptions{code-for-first-row = \scriptstyle \color{blue},
                   	   code-for-first-col = \scriptstyle \color{blue}
	}
	\begin{bNiceArray}{*{3}{c}|*{3}{c}|*{3}{c}}[first-row,first-col]
				& 1,1 	& 1,2	& 1,3	& 2,1 	& 2,2	& 2,3	& 3,1 	& 3,2	& 3,3	\\
		1		& 1	& 1	& 1	& 0	& 0	& 0	& 0	& 0	& 0	\\
		2		& 0	& 0	& 0	& 1	& 1	& 1	& 0	& 0	& 0	\\
		3		& 0	& 0	& 0	& 0	& 0	& 0	& 0	& 0	& 0
	\end{bNiceArray}
	$}
	\; \text{;} \;
	\scalebox{0.75}{
	$
	\NiceMatrixOptions{code-for-first-row = \scriptstyle \color{blue},
                   	   code-for-first-col = \scriptstyle \color{blue}
	}
	\begin{bNiceArray}{*{3}{c}|*{3}{c}|*{3}{c}}[first-row,first-col]
				& 1,1 	& 1,2	& 1,3	& 2,1 	& 2,2	& 2,3	& 3,1 	& 3,2	& 3,3	\\
		1		& 0	& 1	& 0	& 0	& 0	& 0	& 0	& 0	& 0	\\
		2		& 0	& 0	& 0	& 1	& 0	& 0	& 0	& 0	& 0	\\
		3		& 0	& 0	& 0	& 0	& 0	& 0	& 0	& 0	& 0
	\end{bNiceArray}
	$}
	\; \text{;} \;
	\scalebox{0.75}{
	$
	\NiceMatrixOptions{code-for-first-row = \scriptstyle \color{blue},
                   	   code-for-first-col = \scriptstyle \color{blue}
	}
	\begin{bNiceArray}{*{3}{c}|*{3}{c}|*{3}{c}}[first-row,first-col]
				& 1,1 	& 1,2	& 1,3	& 2,1 	& 2,2	& 2,3	& 3,1 	& 3,2	& 3,3	\\
		1		& 1	& 0	& 0	& 0	& 0	& 0	& 0	& 0	& 0	\\
		2		& 0	& 0	& 0	& 0	& 1	& 0	& 0	& 0	& 0	\\
		3		& 0	& 0	& 0	& 0	& 0	& 0	& 0	& 0	& 0
	\end{bNiceArray}
	$}
\end{equation}
\begin{equation}
	\scalebox{0.75}{
	$
	\NiceMatrixOptions{code-for-first-row = \scriptstyle \color{blue},
                   	   code-for-first-col = \scriptstyle \color{blue}
	}
	\begin{bNiceArray}{*{3}{c}|*{3}{c}|*{3}{c}}[first-row,first-col]
				& 1,1 	& 1,2	& 1,3	& 2,1 	& 2,2	& 2,3	& 3,1 	& 3,2	& 3,3	\\
		1		& 0	& 1	& 0	& 0	& 0	& 0	& 0	& 0	& 0	\\
		2		& 0	& 0	& 0	& 1	& 0	& 0	& 0	& 0	& 0	\\
		3		& 0	& 0	& 0	& 0	& 0	& 0	& 0	& 0	& 0
	\end{bNiceArray}
	$}
	\; \text{;} \;
	\scalebox{0.75}{
	$
	\NiceMatrixOptions{code-for-first-row = \scriptstyle \color{blue},
                   	   code-for-first-col = \scriptstyle \color{blue}
	}
	\begin{bNiceArray}{*{3}{c}|*{3}{c}|*{3}{c}}[first-row,first-col]
				& 1,1 	& 1,2	& 1,3	& 2,1 	& 2,2	& 2,3	& 3,1 	& 3,2	& 3,3	\\
		1		& 0	& 1	& 0	& 0	& 0	& 0	& 0	& 0	& 0	\\
		2		& 0	& 0	& 0	& 1	& 0	& 0	& 0	& 0	& 0	\\
		3		& 0	& 0	& 0	& 0	& 0	& 0	& 0	& 0	& 0
	\end{bNiceArray}
	$}
	\; \text{;} \;
	\scalebox{0.75}{
	$
	\NiceMatrixOptions{code-for-first-row = \scriptstyle \color{blue},
                   	   code-for-first-col = \scriptstyle \color{blue}
	}
	\begin{bNiceArray}{*{3}{c}|*{3}{c}|*{3}{c}}[first-row,first-col]
				& 1,1 	& 1,2	& 1,3	& 2,1 	& 2,2	& 2,3	& 3,1 	& 3,2	& 3,3	\\
		1		& 1	& 0	& 0	& 0	& 0	& 0	& 0	& 0	& 0	\\
		2		& 0	& 0	& 0	& 0	& 1	& 0	& 0	& 0	& 0	\\
		3		& 0	& 0	& 0	& 0	& 0	& 0	& 0	& 0	& 0
	\end{bNiceArray}
	$}
\end{equation}
\begin{equation}
	\scalebox{0.75}{
	$
	\NiceMatrixOptions{code-for-first-row = \scriptstyle \color{blue},
                   	   code-for-first-col = \scriptstyle \color{blue}
	}
	\begin{bNiceArray}{*{3}{c}|*{3}{c}|*{3}{c}}[first-row,first-col]
				& 1,1 	& 1,2	& 1,3	& 2,1 	& 2,2	& 2,3	& 3,1 	& 3,2	& 3,3	\\
		1		& 1	& 0	& 0	& 0	& 0	& 0	& 0	& 0	& 0	\\
		2		& 0	& 0	& 0	& 0	& 1	& 0	& 0	& 0	& 0	\\
		3		& 0	& 0	& 0	& 0	& 0	& 0	& 0	& 0	& 0
	\end{bNiceArray}
	$}
	\; \text{;} \;
	\scalebox{0.75}{
	$
	\NiceMatrixOptions{code-for-first-row = \scriptstyle \color{blue},
                   	   code-for-first-col = \scriptstyle \color{blue}
	}
	\begin{bNiceArray}{*{3}{c}|*{3}{c}|*{3}{c}}[first-row,first-col]
				& 1,1 	& 1,2	& 1,3	& 2,1 	& 2,2	& 2,3	& 3,1 	& 3,2	& 3,3	\\
		1		& 0	& 0	& 0	& 0	& 0	& 0	& 0	& 0	& 0	\\
		2		& 0	& 0	& 0	& 0	& 0	& 0	& 0	& 0	& 0	\\
		3		& 0	& 0	& 0	& 0	& 0	& 0	& 0	& 0	& 0
	\end{bNiceArray}
	$}
	\; \text{;} \;
	\scalebox{0.75}{
	$
	\NiceMatrixOptions{code-for-first-row = \scriptstyle \color{blue},
                   	   code-for-first-col = \scriptstyle \color{blue}
	}
	\begin{bNiceArray}{*{3}{c}|*{3}{c}|*{3}{c}}[first-row,first-col]
				& 1,1 	& 1,2	& 1,3	& 2,1 	& 2,2	& 2,3	& 3,1 	& 3,2	& 3,3	\\
		1		& 1	& 0	& 0	& 0	& 0	& 0	& 0	& 0	& 0	\\
		2		& 0	& 0	& 0	& 0	& 1	& 0	& 0	& 0	& 0	\\
		3		& 0	& 0	& 0	& 0	& 0	& 0	& 0	& 0	& 0
	\end{bNiceArray}
	$}
\end{equation}

\textbf{Step 3:} We calculate all possible $(3,0)$--bilabelled graph diagrams that
have only external red edges between red vertices from the five bilabelled graph diagrams
given in Step $1$.

For $\bm{A_0}$, as the number of red vertices is $1$, we create all length one strings
having entries in $0$ to $m=1$, removing the all $0$ string. As a result, we create a new
$(3,0)$--bilabelled graph diagram from the string $1$ by adding $1$ new red edge outwards from
the red vertex in $\bm{A_0}$, adding in a new red vertex to make this new 
red edge possible.
Hence we obtain
\begin{equation} \label{S2step3externaledges1}
	\begin{aligned}
		\bm{A_{2}} = \scalebox{0.6}{\input{graph21External1A0.tikz}}
	\end{aligned}
\end{equation}
which, after Frobenius duality, corresponds to the $A$-homomorphism matrix
\begin{equation}
	\scalebox{0.75}{
	$
	\NiceMatrixOptions{code-for-first-row = \scriptstyle \color{blue},
                   	   code-for-first-col = \scriptstyle \color{blue}
	}
	\begin{bNiceArray}{*{3}{c}|*{3}{c}|*{3}{c}}[first-row,first-col]
				& 1,1 	& 1,2	& 1,3	& 2,1 	& 2,2	& 2,3	& 3,1 	& 3,2	& 3,3	\\
		1		& 1	& 0	& 0	& 0	& 0	& 0	& 0	& 0	& 0	\\
		2		& 0	& 0	& 0	& 0	& 1	& 0	& 0	& 0	& 0	\\
		3		& 0	& 0	& 0	& 0	& 0	& 0	& 0	& 0	& 0
	\end{bNiceArray}
	$}
\end{equation}

However, for $\bm{B_0}$, $\bm{C_0}$ and $\bm{D_0}$, 
calling these generically by $\bm{H}$, we know that
the number of red vertices, $c$, in $\bm{H}$ is $2$. 
As a result, we create all length two strings
having entries in $0$ to $m=1$, removing the all $0$ string.
There are three such strings, $01$, $10$ and $11$.
Hence we create three new $(3,0)$--bilabelled graph diagrams using these strings
from each of $\bm{B_0}$, $\bm{C_0}$ and $\bm{D_0}$,
by adding $1$ new red edge outwards from a red vertex that corresponds to a $1$ in
the string, adding in a new red vertex to make this new red edge possible.
They are given by
\begin{equation} \label{S2step3externaledges2}
	\begin{aligned}
		\bm{B_{2,1}} = \scalebox{0.6}{\input{graph21External1B0.tikz}}
		\; \text{;} \;
		\bm{B_{2,2}} = \scalebox{0.6}{\input{graph21External2B0.tikz}}
		\; \text{;} \;
		\bm{B_{2,3}} = \scalebox{0.6}{\input{graph21External3B0.tikz}}
	\end{aligned}
\end{equation}

\begin{equation} \label{S2step3externaledges3}
	\begin{aligned}
		\bm{C_{2,1}} = \scalebox{0.6}{\input{graph21External1C0.tikz}}
		\; \text{;} \;
		\bm{C_{2,2}} = \scalebox{0.6}{\input{graph21External2C0.tikz}}
		\; \text{;} \;
		\bm{C_{2,3}} = \scalebox{0.6}{\input{graph21External3C0.tikz}}
	\end{aligned}
\end{equation}

\begin{equation} \label{S2step3externaledges4}
	\begin{aligned}
		\bm{D_{2,1}} = \scalebox{0.6}{\input{graph21External1D0.tikz}}
		\; \text{;} \;
		\bm{D_{2,2}} = \scalebox{0.6}{\input{graph21External2D0.tikz}}
		\; \text{;} \;
		\bm{D_{2,3}} = \scalebox{0.6}{\input{graph21External3D0.tikz}}
	\end{aligned}
\end{equation}
which, after Frobenius duality, correspond to the $A$-homomorphism matrices
\begin{equation}
	\scalebox{0.75}{
	$
	\NiceMatrixOptions{code-for-first-row = \scriptstyle \color{blue},
                   	   code-for-first-col = \scriptstyle \color{blue}
	}
	\begin{bNiceArray}{*{3}{c}|*{3}{c}|*{3}{c}}[first-row,first-col]
				& 1,1 	& 1,2	& 1,3	& 2,1 	& 2,2	& 2,3	& 3,1 	& 3,2	& 3,3	\\
		1		& 1	& 1	& 0	& 0	& 0	& 0	& 0	& 0	& 0	\\
		2		& 0	& 0	& 0	& 1	& 1	& 0	& 0	& 0	& 0	\\
		3		& 0	& 0	& 0	& 0	& 0	& 0	& 1	& 1	& 0
	\end{bNiceArray}
	$}
	\; \text{;} \;
	\scalebox{0.75}{
	$
	\NiceMatrixOptions{code-for-first-row = \scriptstyle \color{blue},
                   	   code-for-first-col = \scriptstyle \color{blue}
	}
	\begin{bNiceArray}{*{3}{c}|*{3}{c}|*{3}{c}}[first-row,first-col]
				& 1,1 	& 1,2	& 1,3	& 2,1 	& 2,2	& 2,3	& 3,1 	& 3,2	& 3,3	\\
		1		& 1	& 1	& 1	& 0	& 0	& 0	& 0	& 0	& 0	\\
		2		& 0	& 0	& 0	& 1	& 1	& 1	& 0	& 0	& 0	\\
		3		& 0	& 0	& 0	& 0	& 0	& 0	& 0	& 0	& 0
	\end{bNiceArray}
	$}
	\; \text{;} \;
	\scalebox{0.75}{
	$
	\NiceMatrixOptions{code-for-first-row = \scriptstyle \color{blue},
                   	   code-for-first-col = \scriptstyle \color{blue}
	}
	\begin{bNiceArray}{*{3}{c}|*{3}{c}|*{3}{c}}[first-row,first-col]
				& 1,1 	& 1,2	& 1,3	& 2,1 	& 2,2	& 2,3	& 3,1 	& 3,2	& 3,3	\\
		1		& 1	& 1	& 0	& 0	& 0	& 0	& 0	& 0	& 0	\\
		2		& 0	& 0	& 0	& 1	& 1	& 0	& 0	& 0	& 0	\\
		3		& 0	& 0	& 0	& 0	& 0	& 0	& 0	& 0	& 0
	\end{bNiceArray}
	$}
\end{equation}
\begin{equation}
	\scalebox{0.75}{
	$
	\NiceMatrixOptions{code-for-first-row = \scriptstyle \color{blue},
                   	   code-for-first-col = \scriptstyle \color{blue}
	}
	\begin{bNiceArray}{*{3}{c}|*{3}{c}|*{3}{c}}[first-row,first-col]
				& 1,1 	& 1,2	& 1,3	& 2,1 	& 2,2	& 2,3	& 3,1 	& 3,2	& 3,3	\\
		1		& 1	& 0	& 0	& 1	& 0	& 0	& 0	& 0	& 0	\\
		2		& 0	& 1	& 0	& 0	& 1	& 0	& 0	& 0	& 0	\\
		3		& 0	& 0	& 1	& 0	& 0	& 1	& 0	& 0	& 0
	\end{bNiceArray}
	$}
	\; \text{;} \;
	\scalebox{0.75}{
	$
	\NiceMatrixOptions{code-for-first-row = \scriptstyle \color{blue},
                   	   code-for-first-col = \scriptstyle \color{blue}
	}
	\begin{bNiceArray}{*{3}{c}|*{3}{c}|*{3}{c}}[first-row,first-col]
				& 1,1 	& 1,2	& 1,3	& 2,1 	& 2,2	& 2,3	& 3,1 	& 3,2	& 3,3	\\
		1		& 1	& 0	& 0	& 1	& 0	& 0	& 1	& 0	& 0	\\
		2		& 0	& 1	& 0	& 0	& 1	& 0	& 0	& 1	& 0	\\
		3		& 0	& 0	& 0	& 0	& 0	& 0	& 0	& 0	& 0
	\end{bNiceArray}
	$}
	\; \text{;} \;
	\scalebox{0.75}{
	$
	\NiceMatrixOptions{code-for-first-row = \scriptstyle \color{blue},
                   	   code-for-first-col = \scriptstyle \color{blue}
	}
	\begin{bNiceArray}{*{3}{c}|*{3}{c}|*{3}{c}}[first-row,first-col]
				& 1,1 	& 1,2	& 1,3	& 2,1 	& 2,2	& 2,3	& 3,1 	& 3,2	& 3,3	\\
		1		& 1	& 0	& 0	& 1	& 0	& 0	& 0	& 0	& 0	\\
		2		& 0	& 1	& 0	& 0	& 1	& 0	& 0	& 0	& 0	\\
		3		& 0	& 0	& 0	& 0	& 0	& 0	& 0	& 0	& 0
	\end{bNiceArray}
	$}
\end{equation}
\begin{equation}
	\scalebox{0.75}{
	$
	\NiceMatrixOptions{code-for-first-row = \scriptstyle \color{blue},
                   	   code-for-first-col = \scriptstyle \color{blue}
	}
	\begin{bNiceArray}{*{3}{c}|*{3}{c}|*{3}{c}}[first-row,first-col]
				& 1,1 	& 1,2	& 1,3	& 2,1 	& 2,2	& 2,3	& 3,1 	& 3,2	& 3,3	\\
		1		& 1	& 0	& 0	& 0	& 1	& 0	& 0	& 0	& 1	\\
		2		& 1	& 0	& 0	& 0	& 1	& 0	& 0	& 0	& 1	\\
		3		& 0	& 0	& 0	& 0	& 0	& 0	& 0	& 0	& 0
	\end{bNiceArray}
	$}
	\; \text{;} \;
	\scalebox{0.75}{
	$
	\NiceMatrixOptions{code-for-first-row = \scriptstyle \color{blue},
                   	   code-for-first-col = \scriptstyle \color{blue}
	}
	\begin{bNiceArray}{*{3}{c}|*{3}{c}|*{3}{c}}[first-row,first-col]
				& 1,1 	& 1,2	& 1,3	& 2,1 	& 2,2	& 2,3	& 3,1 	& 3,2	& 3,3	\\
		1		& 1	& 0	& 0	& 0	& 1	& 0	& 0	& 0	& 0	\\
		2		& 1	& 0	& 0	& 0	& 1	& 0	& 0	& 0	& 0	\\
		3		& 1	& 0	& 0	& 0	& 1	& 0	& 0	& 0	& 0
	\end{bNiceArray}
	$}
	\; \text{;} \;
	\scalebox{0.75}{
	$
	\NiceMatrixOptions{code-for-first-row = \scriptstyle \color{blue},
                   	   code-for-first-col = \scriptstyle \color{blue}
	}
	\begin{bNiceArray}{*{3}{c}|*{3}{c}|*{3}{c}}[first-row,first-col]
				& 1,1 	& 1,2	& 1,3	& 2,1 	& 2,2	& 2,3	& 3,1 	& 3,2	& 3,3	\\
		1		& 1	& 0	& 0	& 0	& 1	& 0	& 0	& 0	& 0	\\
		2		& 1	& 0	& 0	& 0	& 1	& 0	& 0	& 0	& 0	\\
		3		& 0	& 0	& 0	& 0	& 0	& 0	& 0	& 0	& 0
	\end{bNiceArray}
	$}
\end{equation}

For $\bm{E_0}$, we know that
the number of red vertices, $c$, in $\bm{E_0}$ is $3$.
As a result, we create all length three strings
having entries in $0$ to $m=1$, removing the all $0$ string.
There are seven such strings, $001$, $010$, $011$, $100$, $101$, $110$ and $111$.
Hence we create seven new $(3,0)$--bilabelled graph diagrams using these strings
from $\bm{E_0}$, by adding $1$ new red edge outwards from a red vertex 
that corresponds to a $1$ in
the string, adding in a new red vertex to make this new red edge possible.
They are given by

\begin{equation} \label{S2step3externaledges5}
	\begin{aligned}
		\bm{E_{2,1}} = \scalebox{0.6}{\input{graph21External1E0.tikz}}
		\; \text{;} \;
		\bm{E_{2,2}} = \scalebox{0.6}{\input{graph21External2E0.tikz}}
		\; \text{;} \;
		\bm{E_{2,3}} = \scalebox{0.6}{\input{graph21External3E0.tikz}}
		\; \text{;} \;
		\bm{E_{2,4}} = \scalebox{0.6}{\input{graph21External4E0.tikz}}
	\end{aligned}
\end{equation}

\begin{equation} \label{S2step3externaledges5}
	\begin{aligned}
		\bm{E_{2,5}} = \scalebox{0.6}{\input{graph21External5E0.tikz}}
		\; \text{;} \;
		\bm{E_{2,6}} = \scalebox{0.6}{\input{graph21External6E0.tikz}}
		\; \text{;} \;
		\bm{E_{2,7}} = \scalebox{0.6}{\input{graph21External7E0.tikz}}
	\end{aligned}
\end{equation}

which, after Frobenius duality, correspond to the $A$-homomorphism matrices

\begin{equation}
	\scalebox{0.75}{
	$
	\NiceMatrixOptions{code-for-first-row = \scriptstyle \color{blue},
                   	   code-for-first-col = \scriptstyle \color{blue}
	}
	\begin{bNiceArray}{*{3}{c}|*{3}{c}|*{3}{c}}[first-row,first-col]
				& 1,1 	& 1,2	& 1,3	& 2,1 	& 2,2	& 2,3	& 3,1 	& 3,2	& 3,3	\\
		1		& 1	& 1	& 0	& 1	& 1	& 0	& 1	& 1	& 0	\\
		2		& 1	& 1	& 0	& 1	& 1	& 0	& 1	& 1	& 0	\\
		3		& 1	& 1	& 0	& 1	& 1	& 0	& 1	& 1	& 0
	\end{bNiceArray}
	$}
	\; \text{;} \;
	\scalebox{0.75}{
	$
	\NiceMatrixOptions{code-for-first-row = \scriptstyle \color{blue},
                   	   code-for-first-col = \scriptstyle \color{blue}
	}
	\begin{bNiceArray}{*{3}{c}|*{3}{c}|*{3}{c}}[first-row,first-col]
				& 1,1 	& 1,2	& 1,3	& 2,1 	& 2,2	& 2,3	& 3,1 	& 3,2	& 3,3	\\
		1		& 1	& 1	& 1	& 1	& 1	& 1	& 0	& 0	& 0	\\
		2		& 1	& 1	& 1	& 1	& 1	& 1	& 0	& 0	& 0	\\
		3		& 1	& 1	& 1	& 1	& 1	& 1	& 0	& 0	& 0
	\end{bNiceArray}
	$}
\end{equation}
\begin{equation}
	\scalebox{0.75}{
	$
	\NiceMatrixOptions{code-for-first-row = \scriptstyle \color{blue},
                   	   code-for-first-col = \scriptstyle \color{blue}
	}
	\begin{bNiceArray}{*{3}{c}|*{3}{c}|*{3}{c}}[first-row,first-col]
				& 1,1 	& 1,2	& 1,3	& 2,1 	& 2,2	& 2,3	& 3,1 	& 3,2	& 3,3	\\
		1		& 1	& 1	& 0	& 1	& 1	& 0	& 0	& 0	& 0	\\
		2		& 1	& 1	& 0	& 1	& 1	& 0	& 0	& 0	& 0	\\
		3		& 1	& 1	& 0	& 1	& 1	& 0	& 0	& 0	& 0
	\end{bNiceArray}
	$}
	\; \text{;} \;
	\scalebox{0.75}{
	$
	\NiceMatrixOptions{code-for-first-row = \scriptstyle \color{blue},
                   	   code-for-first-col = \scriptstyle \color{blue}
	}
	\begin{bNiceArray}{*{3}{c}|*{3}{c}|*{3}{c}}[first-row,first-col]
				& 1,1 	& 1,2	& 1,3	& 2,1 	& 2,2	& 2,3	& 3,1 	& 3,2	& 3,3	\\
		1		& 1	& 1	& 1	& 1	& 1	& 1	& 1	& 1	& 1	\\
		2		& 1	& 1	& 1	& 1	& 1	& 1	& 1	& 1	& 1	\\
		3		& 0	& 0	& 0	& 0	& 0	& 0	& 0	& 0	& 0
	\end{bNiceArray}
	$}
\end{equation}
\begin{equation}
	\scalebox{0.75}{
	$
	\NiceMatrixOptions{code-for-first-row = \scriptstyle \color{blue},
                   	   code-for-first-col = \scriptstyle \color{blue}
	}
	\begin{bNiceArray}{*{3}{c}|*{3}{c}|*{3}{c}}[first-row,first-col]
				& 1,1 	& 1,2	& 1,3	& 2,1 	& 2,2	& 2,3	& 3,1 	& 3,2	& 3,3	\\
		1		& 1	& 1	& 0	& 1	& 1	& 0	& 1	& 1	& 0	\\
		2		& 1	& 1	& 0	& 1	& 1	& 0	& 1	& 1	& 0	\\
		3		& 0	& 0	& 0	& 0	& 0	& 0	& 0	& 0	& 0
	\end{bNiceArray}
	$}
	\; \text{;} \;
	\scalebox{0.75}{
	$
	\NiceMatrixOptions{code-for-first-row = \scriptstyle \color{blue},
                   	   code-for-first-col = \scriptstyle \color{blue}
	}
	\begin{bNiceArray}{*{3}{c}|*{3}{c}|*{3}{c}}[first-row,first-col]
				& 1,1 	& 1,2	& 1,3	& 2,1 	& 2,2	& 2,3	& 3,1 	& 3,2	& 3,3	\\
		1		& 1	& 1	& 1	& 1	& 1	& 1	& 0	& 0	& 0	\\
		2		& 1	& 1	& 1	& 1	& 1	& 1	& 0	& 0	& 0	\\
		3		& 0	& 0	& 0	& 0	& 0	& 0	& 0	& 0	& 0
	\end{bNiceArray}
	$}
	\; \text{;} \;
		\scalebox{0.75}{
	$
	\NiceMatrixOptions{code-for-first-row = \scriptstyle \color{blue},
                   	   code-for-first-col = \scriptstyle \color{blue}
	}
	\begin{bNiceArray}{*{3}{c}|*{3}{c}|*{3}{c}}[first-row,first-col]
				& 1,1 	& 1,2	& 1,3	& 2,1 	& 2,2	& 2,3	& 3,1 	& 3,2	& 3,3	\\
		1		& 1	& 1	& 0	& 1	& 1	& 0	& 0	& 0	& 0	\\
		2		& 1	& 1	& 0	& 1	& 1	& 0	& 0	& 0	& 0	\\
		3		& 0	& 0	& 0	& 0	& 0	& 0	& 0	& 0	& 0
	\end{bNiceArray}
	$}
\end{equation}

\textbf{Step 4:} We calculate all possible $(3,0)$--bilabelled graph diagrams that
have external red edges from the $32$ bilabelled graph diagrams that were given in Step $2$.

None of the bilabelled graph diagrams given in 
(\ref{S2internaledgesi})
and
(\ref{S2internaledgesii})
will lead to new bilabelled graph diagrams, since each red vertex in each bilabelled 
graph diagram is connected with another red vertex in the diagram.

Only six bilabelled graph diagrams in (\ref{S2internaledges000}) - (\ref{S2internaledges120})
give rise to new bilabelled graph diagrams:
they are
$\bm{E_{001}}$,
$\bm{E_{002}}$,
$\bm{E_{010}}$,
$\bm{E_{020}}$,
$\bm{E_{100}}$, and
$\bm{E_{200}}$,
as they are the only bilabelled graph diagrams that have at least one (in this example, exactly one) red vertex that is not connected with any other vertex.

As each bilabelled graph diagram has three red vertices having only a single
red vertex that is not connected with any other vertex, and since $m = 1$,
this step produces exactly one string of length three, where each position
in the string corresponds to a red vertex in the bilabelled graph diagram.

Hence, from 
$\bm{E_{001}}$,
$\bm{E_{002}}$, and
$\bm{E_{010}}$,
we obtain
\begin{equation} \label{S2externalinternaledgesi}
	\begin{aligned}
		\bm{E_{001, E}} = \scalebox{0.6}{\input{graph21ExInternal001.tikz}}
		\; \text{;} \;
		\bm{E_{002, E}} = \scalebox{0.6}{\input{graph21ExInternal002.tikz}}
		\; \text{;} \;
		\bm{E_{010, E}} = \scalebox{0.6}{\input{graph21ExInternal010.tikz}}
	\end{aligned}
\end{equation}

and from 
$\bm{E_{020}}$,
$\bm{E_{100}}$, and
$\bm{E_{200}}$,
we obtain
\begin{equation} \label{S2externalinternaledgesii}
	\begin{aligned}
		\bm{E_{020, E}} = \scalebox{0.6}{\input{graph21ExInternal020.tikz}}
		\; \text{;} \;
		\bm{E_{100, E}} = \scalebox{0.6}{\input{graph21ExInternal100.tikz}}
		\; \text{;} \;
		\bm{E_{200, E}} = \scalebox{0.6}{\input{graph21ExInternal200.tikz}}
	\end{aligned}
\end{equation}

These six bilabelled graph diagrams correspond, after Frobenius duality, to the following $A$-homomorphism matrices.

\begin{equation}
	\scalebox{0.75}{
	$
	\NiceMatrixOptions{code-for-first-row = \scriptstyle \color{blue},
                   	   code-for-first-col = \scriptstyle \color{blue}
	}
	\begin{bNiceArray}{*{3}{c}|*{3}{c}|*{3}{c}}[first-row,first-col]
				& 1,1 	& 1,2	& 1,3	& 2,1 	& 2,2	& 2,3	& 3,1 	& 3,2	& 3,3	\\
		1		& 0	& 1	& 0	& 1	& 0	& 0	& 0	& 0	& 0	\\
		2		& 0	& 1	& 0	& 1	& 0	& 0	& 0	& 0	& 0	\\
		3		& 0	& 0	& 0	& 0	& 0	& 0	& 0	& 0	& 0
	\end{bNiceArray}
	$}
	\; \text{;} \;
	\scalebox{0.75}{
	$
	\NiceMatrixOptions{code-for-first-row = \scriptstyle \color{blue},
                   	   code-for-first-col = \scriptstyle \color{blue}
	}
	\begin{bNiceArray}{*{3}{c}|*{3}{c}|*{3}{c}}[first-row,first-col]
				& 1,1 	& 1,2	& 1,3	& 2,1 	& 2,2	& 2,3	& 3,1 	& 3,2	& 3,3	\\
		1		& 1	& 0	& 0	& 0	& 1	& 0	& 0	& 0	& 0	\\
		2		& 1	& 0	& 0	& 0	& 1	& 0	& 0	& 0	& 0	\\
		3		& 0	& 0	& 0	& 0	& 0	& 0	& 0	& 0	& 0
	\end{bNiceArray}
	$}
	\; \text{;} \;
	\scalebox{0.75}{
	$
	\NiceMatrixOptions{code-for-first-row = \scriptstyle \color{blue},
                   	   code-for-first-col = \scriptstyle \color{blue}
	}
	\begin{bNiceArray}{*{3}{c}|*{3}{c}|*{3}{c}}[first-row,first-col]
				& 1,1 	& 1,2	& 1,3	& 2,1 	& 2,2	& 2,3	& 3,1 	& 3,2	& 3,3	\\
		1		& 0	& 0	& 0	& 0	& 1	& 0	& 0	& 0	& 0	\\
		2		& 1	& 0	& 0	& 0	& 0	& 0	& 0	& 0	& 0	\\
		3		& 0	& 0	& 0	& 0	& 0	& 0	& 0	& 0	& 0
	\end{bNiceArray}
	$}
\end{equation}
\begin{equation}
	\scalebox{0.75}{
	$
	\NiceMatrixOptions{code-for-first-row = \scriptstyle \color{blue},
                   	   code-for-first-col = \scriptstyle \color{blue}
	}
	\begin{bNiceArray}{*{3}{c}|*{3}{c}|*{3}{c}}[first-row,first-col]
				& 1,1 	& 1,2	& 1,3	& 2,1 	& 2,2	& 2,3	& 3,1 	& 3,2	& 3,3	\\
		1		& 1	& 0	& 0	& 0	& 0	& 0	& 0	& 0	& 0	\\
		2		& 0	& 0	& 0	& 0	& 1	& 0	& 0	& 0	& 0	\\
		3		& 0	& 0	& 0	& 0	& 0	& 0	& 0	& 0	& 0
	\end{bNiceArray}
	$}
	\; \text{;} \;
	\scalebox{0.75}{
	$
	\NiceMatrixOptions{code-for-first-row = \scriptstyle \color{blue},
                   	   code-for-first-col = \scriptstyle \color{blue}
	}
	\begin{bNiceArray}{*{3}{c}|*{3}{c}|*{3}{c}}[first-row,first-col]
				& 1,1 	& 1,2	& 1,3	& 2,1 	& 2,2	& 2,3	& 3,1 	& 3,2	& 3,3	\\
		1		& 0	& 0	& 0	& 1	& 1	& 0	& 0	& 0	& 0	\\
		2		& 1	& 1	& 0	& 0	& 0	& 0	& 0	& 0	& 0	\\
		3		& 0	& 0	& 0	& 0	& 0	& 0	& 0	& 0	& 0
	\end{bNiceArray}
	$}
	\; \text{;} \;
	\scalebox{0.75}{
	$
	\NiceMatrixOptions{code-for-first-row = \scriptstyle \color{blue},
                   	   code-for-first-col = \scriptstyle \color{blue}
	}
	\begin{bNiceArray}{*{3}{c}|*{3}{c}|*{3}{c}}[first-row,first-col]
				& 1,1 	& 1,2	& 1,3	& 2,1 	& 2,2	& 2,3	& 3,1 	& 3,2	& 3,3	\\
		1		& 1	& 1	& 0	& 0	& 0	& 0	& 0	& 0	& 0	\\
		2		& 0	& 0	& 0	& 1	& 1	& 0	& 0	& 0	& 0	\\
		3		& 0	& 0	& 0	& 0	& 0	& 0	& 0	& 0	& 0
	\end{bNiceArray}
	$}
\end{equation}

\textbf{Step 5:} There are no new bilabelled graph diagrams to consider in this step as $A$ does not have any loops. 

In total, this gives us $60$ $(3,0)$--bilabelled graph diagrams that are appropriate for $A$.

We have implicitly applied \textbf{Step 6} in calculating the $A$-homomorphism matrices
during Steps 1 through 5 inclusive since we had to convert 
each $(3,0)$--bilabelled graph diagram into a 
$(2,1)$--bilabelled graph diagram first before applying the functor $\mathcal{F}^{G}$ 
to obtain each $A$-homomorphism matrix.
Hence we move onto the second half of \textbf{Step 7}:
we remove all duplicate matrices and the all zero matrices.
We see that we are left with $31$ matrices in the spanning set.
After this, 
to obtain the weight matrix for an $\Aut(A) \cong S_2$-equivariant linear
layer function from 
$(\mathbb{R}^3)^{\otimes 2}$ to 
$\mathbb{R}^3$,
we would weight each of the $31$ matrices in the spanning set and then add them together.

%


%

It is important to highlight that the vector space 
$
	\Hom((\mathbb{R}^{3})^{\otimes 2},
	\mathbb{R}^{3})
$,
in which
$
	\Hom_{S_2}((\mathbb{R}^{3})^{\otimes 2},
	\mathbb{R}^{3})
$ 
lives, is of dimension $27$; hence, there must be linear dependencies amongst the $31$
spanning set elements that we have found.
However, determining these linear dependencies in general 
at the graph level 
\textit{a priori}
is left for future work.

We can also find the elements of the spanning set for
$
	\Hom_{S_2}((\mathbb{R}^{3})^{\otimes 2},
	\mathbb{R}^{3})
$
for the other two embeddings of $S_2$ in $S_3$, 
given by $\Aut(B)$ and $\Aut(C)$.
For $\Aut(B)$, we perform the permutation $(13)$ on each index of the rows and columns 
of the spanning set matrices found for $\Aut(A)$, and for $\Aut(C)$,
we perform the permutation $(23)$ instead.

\end{example}

\section{Adding Features and Biases}

\subsection{Features} 

We have assumed throughout that the feature dimension for all of the layers
that appear in a graph automorphism group equivariant neural network is one.  
We can adapt all of the results that
we have shown for the case where the feature dimension of the layers is
greater than one.

Let $G$ be a graph having $n$ vertices, and suppose that an $r$-order tensor
has a feature space of dimension $d_r$.  We now wish to find a spanning set for
\begin{equation} \label{HomGraphfeatures}
	\Hom_{\Aut(G)}((\mathbb{R}^{n})^{\otimes k} \otimes \mathbb{R}^{d_k},
(\mathbb{R}^{n})^{\otimes l} \otimes \mathbb{R}^{d_l}) 
\end{equation} 
in the standard basis of $\mathbb{R}^{n}$. 

The spanning set can be found by adapting the result given in
(\ref{autspanres}), and is \begin{equation} \{X_{\bm{H}, i, j}^G \mid [\bm{H}]
\in \mathcal{G}(k,l), i \in [d_l], j \in [d_k]\} \end{equation} where now, if
$\bm{H} \coloneqq (H, \bm{k}, \bm{l})$, then $X_{\bm{H}, i, j}^G$ is the $(n^l
\times d_l) \times (n^k \times d_k)$ matrix that has $(I,i,J,j)$--entry given
by the number of graph homomorphisms from $H$ to $G$ such that $\bm{l}$ is
mapped to $I$ and $\bm{k}$ is mapped to $J$, and is $0$ otherwise.

\subsection{Biases} 

Including bias terms in the layer functions of an $\Aut(G)$-equivariant neural
network is also possible.  If we consider a learnable linear layer in
$\Hom_{\Aut(G)}((\mathbb{R}^{n})^{\otimes k}, (\mathbb{R}^{n})^{\otimes l})$,
\citet{pearcecrump} shows that the $\Aut(G)$-equivariance of the bias function,
$\beta : ((\mathbb{R}^{n})^{\otimes k}, \rho_k) \rightarrow
((\mathbb{R}^{n})^{\otimes l}, \rho_l)$, needs to satisfy 
\begin{equation} \label{AutGbiasequiv} 
	c = \rho_l(g)c 
\end{equation} 
for all $g \in \Aut(G)$ and for all $c \in (\mathbb{R}^{n})^{\otimes l}$.

Since any $c \in (\mathbb{R}^{n})^{\otimes l}$ satisfying (\ref{AutGbiasequiv})
can be viewed as an element of $\Hom_{\Aut(G)}(\mathbb{R},
(\mathbb{R}^{n})^{\otimes l})$, to find the matrix form of $c$, all we need to
do is to find a spanning set for $\Hom_{\Aut(G)}(\mathbb{R},
(\mathbb{R}^{n})^{\otimes l})$.

But this is simply a matter of applying Theorem \ref{graphmainthrm}, setting $k
= 0$.


\end{document}